\pdfoutput=1

\documentclass[11pt,letterpaper]{article}
\PassOptionsToPackage{table}{xcolor}
\usepackage[utf8]{inputenc}
\usepackage{url}
\usepackage{nicefrac}
\usepackage{microtype}
\usepackage{wrapfig}
\usepackage{listings}
\usepackage[margin=1in]{geometry}
\usepackage[english]{babel}
\usepackage[T1]{fontenc}
\usepackage{parskip}

\usepackage{amsmath}
\usepackage{amssymb}
\usepackage{amsthm}
\usepackage{bm}
\usepackage{booktabs}
\usepackage{makecell}
\usepackage[table]{xcolor}
\usepackage{graphicx}
\usepackage{subcaption}
\usepackage{tikz}
\usetikzlibrary{arrows.meta}
\usepackage{algorithm}
\usepackage{algpseudocode}
\usepackage{thm-restate}
\usepackage[numbers]{natbib}
\usepackage{xspace}
\usepackage[
    colorlinks=true,
    linkcolor=blue!70!black,
    citecolor=blue!70!black,
    urlcolor=blue!70!black
]{hyperref}
\usepackage[nameinlink,noabbrev]{cleveref}

\graphicspath{{./figures/}}
\newtheorem{lemma}{Lemma}

\newtheorem{definition}{Definition}

\newtheorem{assumption}{Assumption}

\DeclareMathOperator*{\argmin}{arg\,min}

\newcommand{\defeq}{:=}
\newcommand{\q}{p}
\newcommand{\apxnu}{\widetilde \nu}
\newcommand{\MDM}{\textup{MDM}}
\newcommand{\voc}{\mathcal V}
\newcommand{\bas}[1]{\begin{align*}#1\end{align*}}
\newcommand{\apx}{\widetilde{X}}
\newcommand{\bb}[1]{\left(#1\right)}
\newcommand{\1}{\mathbf{1}}
\newcommand{\E}{\mathbb{E}}
\newcommand{\trum}{\textrm{TR}\xspace}
\newcommand{\teum}{\textrm{TE}\xspace}
\newcommand{\tp}{\nu}
\newcommand{\tq}{p}
\newcommand{\cL}{\mathcal{L}}
\newcommand{\mask}{\ensuremath{\mathtt{[MASK]}}}
\newcommand{\coh}{\mathrm{coh}}
\newcommand{\TC}{\mathrm{TC}}
\newcommand{\KL}{\mathrm{KL}}
\newcommand{\br}[1]{\left( #1 \right)}
\newcommand{\bc}[1]{\left\{ #1 \right\}}

\newcommand{\bk}{\color{black}}
\makeatletter
\providecommand{\theHALG@line}{}
\renewcommand{\theHALG@line}{\thealgorithm.\arabic{ALG@line}}
\makeatother

\title{Walk fast but be careful: Understanding\\ Parallel Sampling in Masked Diffusion}
\date{}
\author{
Vansh Bansal\thanks{UT Austin SDS, \texttt{vansh@utexas.edu}. Equal contribution.} \and
Cholyeon Cho\thanks{UT Austin CS, \texttt{cc77837@my.utexas.edu }. Equal contribution.}\and
Syamantak Kumar\thanks{UT Austin CS, \texttt{syamantak@utexas.edu}. Equal contribution.} \and
Sujay Sanghavi\thanks{UT Austin ECE, \texttt{sanghavi@mail.utexas.edu}}\and
Purnamrita Sarkar\thanks{UT Austin SDS, \texttt{purna.sarkar@austin.utexas.edu}}
}
\hypersetup{pdftitle={Walk fast but be careful: Understanding Parallel Sampling in Masked Diffusion},pdfauthor={Vansh Bansal, Cholyeon Cho, Syamantak Kumar, Sujay Sanghavi, Purnamrita Sarkar}}

\begin{document}
\maketitle
\begin{abstract}
In this paper, we use random walks on graphs as a verifiable sandbox for studying parallel sampling strategies in masked diffusion models ($\MDM$s). We train an $\MDM$ on random walk samples from a fixed graph. The graph and transition kernel are never shown to the model and serve as latent structure that is both controllable and enables evaluation.  The framework provides a validity check for generated walks and a measure of distributional fidelity through the estimated transition kernel. Using simple graphs, we theoretically prove that parallel unmasking via widely used scores such as lowest entropy is \textit{not} uniformly better than random parallel sampling; even with exact conditional probabilities, performance critically depends on the conditional dependence structure induced by the graph, a phenomenon difficult to isolate in benchmarks like Sudoku. We also develop training-free \textit{bisection samplers} for MDMs, which take logarithmically many steps in the sequence length and are provably exact for random walks if the learned marginals are exact. Experiments on graph-walk tasks confirm that different parallel samplers perform better on different graph structures. Experiments on pretrained $\MDM$s show that bisection-style samplers provide strong speed-quality tradeoffs on OpenWebText generation and reasoning benchmarks including GSM8K, MBPP, and HumanEval. Together, these results use graph walks to uncover conditional dependence as a key principle of parallel MDM sampling and translate this insight into efficient samplers that transfer to language generation and reasoning.\bk
\end{abstract}

\vspace{-1em}
\section{Introduction}
\label{sec:introduction}


Masked diffusion models ($\MDM$s) generate discrete data by iteratively denoising masked tokens
\citep{devlin2019bert,austin2021structured,campbell2022continuous,sahoo2024simple,shi2024simplified,lou2024sedd,ou2025absorbing}.
Because the denoiser can be queried on arbitrary partial contexts, inference is inherently any-order: tokens may be revealed left-to-right, adaptively, or in parallel. This flexibility makes parallel sampling particularly attractive: multiple tokens can be revealed from a single denoiser call, substantially reducing the number of function evaluations (NFEs). Recent work has therefore explored multi-token samplers and adaptive policies for deciding which coordinates to reveal together, with token ordering strongly affecting generation quality on structured tasks such as Sudoku
\citep{chang2022maskgit,shih2022anyorder,kim2025train,benhamu2025accelerated,anari2026parallel,wu2025fastdllm,wei2025slowfast}.

However, parallel speedup can incur statistical error.  With exact conditional marginals, any one-coordinate reveal order is exact by the chain rule, whereas a parallel update replaces the true block conditional by a product of one-coordinate marginals. If blocks of positions ${B_1,\ldots,B_R}$ are revealed over $R$ denoiser calls, the resulting error is bounded by \citep{benhamu2025accelerated}
\begin{align}\label{eq:error-decomposition}
&\KL \br{\tp(X) \bigg|\bigg| \tq_{\theta}(X)} \le
\underbrace{\mathbb{E}\left[\sum_r\sum_{i\in B_r}\Delta_\theta(i;X_{U_r})\right]}_{\Delta_\theta\text{ (marginal approximation error)}}
+
\underbrace{\mathbb{E}\left[\sum_r\TC_\nu(X_{B_r}\mid X_{U_r})\right]}_{\mathrm{TC}\text{ (conditional total correlation)}} .
\end{align}
Here $X=(X_1,\ldots,X_L)\sim\nu$ be a random sequence of length $L$, where $X_i$ is the token at position $i$. For $U\subseteq [L]=\{1,\ldots,L\}$, write $X_U=(X_i)_{i\in U}$ and $x_U=(x_i)_{i\in U}$ for a realization of $X_U$.
Here $U_r$ denotes the positions revealed before round $r$, and $\tq_\theta(X)$ denotes the sampler's output distribution. The expectations are over target samples $X\sim\nu$ and any independent randomness used to select the blocks\footnote{Equality holds if the unmasking policy is a deterministic function of the revealed context $x_U$}: at each round, the sampling strategy selects $B_r$ using only previously revealed values, then reveals the target values $X_{B_r}$. The quantity $\Delta_\theta(i;x_U):=\KL \br{\tp(X_i\mid x_U) \big|\big| \tq_{\theta}(X_i\mid x_U)}$ measures marginal approximation error, while $\TC_\nu(X_B\mid x_U):=\KL \br{\tp(X_B\mid x_U) \big|\big| \prod_{i\in B}\tp(X_i\mid x_U)}$ is the conditional total correlation, measuring the dependence among tokens revealed together. Thus, parallel sampling depends both on marginal prediction accuracy and on the conditional dependence of the selected block.

\begin{figure}[t]
    \captionsetup{skip=5pt}
    \centering    \includegraphics[width=\linewidth]{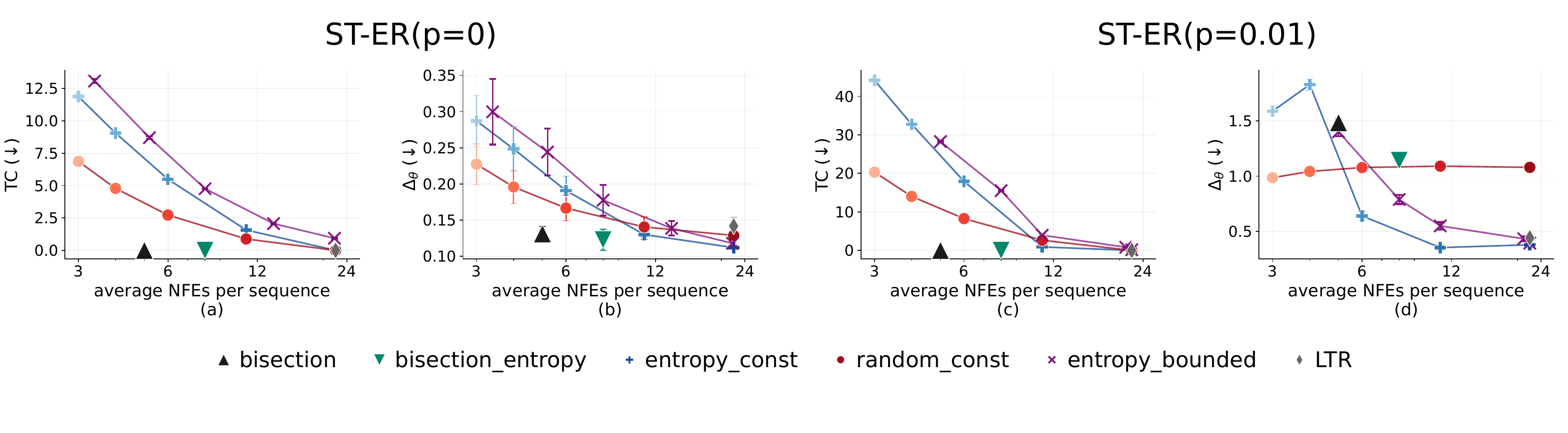}
 \vspace{-2em}
    \caption{
    Error decomposition for endpoint-conditioned generation.
Models are trained only on unconditional random walks; endpoints are revealed at test time.
ST--ER($p$) is a spanning tree with non-tree edges added independently with probability $p$.
We plot the two terms in~(1), total correlation (TC) and marginal approximation error $\Delta_\theta$, against NFEs.
TC is substantially larger than $\Delta_\theta$, and the relative performance of parallel samplers changes with graph structure.
Bisection variants achieve low error with fewer NFEs. The corresponding coherence and TV results are reported in
Figure~\ref{fig:conditional_coherence_appendix}.
    }
    \label{fig:conditional_coherence}
    \vspace{-12pt} 
\end{figure}

The two sources of error are hard to isolate in language, where model error, sampler error, and evaluation noise are intertwined. Sudoku has therefore been used as a controlled benchmark in prior masked-diffusion work. However, commonly used instances with many constraints often have a unique solution \citep{shah2024causallanguagemodelingelicit,kim2025train}. For a unique solution $x^\star$, the intrinsic block-dependence error vanishes for any block $B$:
$\tp(X_B \mid x_U)=\prod_{i\in B}\tp(X_i\mid x_U)=\prod_{i\in B}\mathbf{1}\bc{X_i=x_i^\star}$.
Thus, while such benchmarks can test adaptive unmasking strategies that reduce model approximation error, they do not necessarily test whether a strategy identifies conditionally independent tokens that can be safely unmasked in parallel. In this work we ask,


\emph{How do we mechanistically test when parallel unmasking is statistically accurate, and design parallel samplers that respect the conditional dependence structure of the target distribution?}

To answer these questions, we turn to graph path problems, which are increasingly used as controlled benchmarks for reasoning in language models~\citep{bachmann24a,wang2024alpineunveilingplanningcapability,zhong2025understandingenhancingplanningcapability,huang2026transformerslearnplanmultitoken,ye2025autoregressiondiscretediffusioncomplex,wang2026trainabilitymaskeddiffusionlanguage}. Our goal is complementary to this line of work: we isolate how dependence structure affects parallel sampling. We use MDMs to sample from the distribution over random walks on a graph. This setting gives us direct access to both path validity and distributional fidelity. Exploiting the conditional dependence structure of random walks, we develop an efficient \textit{training-free} parallel sampler, \textit{bisection}. Experiments on unconditional OpenWebText generation and conditional reasoning benchmarks including GSM8K, HumanEval, and MBPP~\citep{cobbe2021training,chen2021evaluating,austin2021program} show that bisection achieves strong speed--quality tradeoffs relative to standard parallel samplers.


We fix a graph $G=(V,E)$ and walk length $L$; training examples are walks $X_{1:L}\in V^L$, with vertices treated as categorical tokens and vocabulary $V\cup\{\mask\}$. The model never sees $G$ or its transition kernel, only samples from the walk distribution, which induces a controllable \textit{latent} dependency structure through graph topology, endpoint constraints, and walk order. Standard random walks are first-order Markov, while higher-order variants allow the next state to depend on the previous $k$ states, capturing memory effects in web navigation, network flows, and complex networks~\citep{chierichetti2012web,rosvall2014memory,benson2016higher,benson2017spacey}.
\bk


Graph walks provide an exact validity check: a generated sequence is coherent iff every consecutive token pair is an edge. Since the ground-truth transition kernel is tractable for both unconditional walks and endpoint-conditioned bridges, they also allow direct measurement of distributional fidelity via TV distance (Section~\ref{sec:exp-metrics}). Figure~\ref{fig:conditional_coherence} shows that small changes in graph structure can reverse sampler rankings. We further isolate the two terms in~\eqref{eq:error-decomposition} (details in Appendix~\ref{app:error-decomposition}). In both graphs, the intrinsic dependence error (TC) is roughly an order of magnitude larger than the marginal approximation error, and sampler rankings largely track TC. This suggests that parallel unmasking depends critically on conditional dependence, not uncertainty alone.

\vspace{-1em}
\paragraph{Contributions.}
Our primary contributions are summarized below.

\textbf{1. A controlled benchmark for parallel masked diffusion.}
    We formulate graph random walks as a controlled benchmark for parallel sampling in MDMs. The model only has sample access to random walks on the graph, but the graph is \textit{available for evaluation}. The benchmark supports unconditional generation, endpoint-conditioned walks, and higher-order random walks, with coherence and TV distance as direct evaluation metrics.


    \textbf{2. Separations between parallel sampling strategies.}
    Simple Directed Acyclic Graph (DAG) constructions show that lowest-entropy two-at-a-time unmasking can outperform random unmasking on some graphs and underperform it on others, demonstrating that no local uncertainty score is uniformly optimal.

    \textbf{3. Bisection samplers for MDMs.}
    We introduce bisection-style samplers that reveal separator positions or separator blocks and then recurse. For order-$k$ Markov walks, the sampler is exact under perfect conditionals for both unconditional walks and endpoint-conditioned walks, with $O(k(1+\log(L/k)))$ NFEs.

    \textbf{4. Transfer beyond graph walks.}
    On pretrained MDMs, bisection improves speed--quality tradeoffs on OpenWebText, GSM8K, MBPP, and HumanEval, showing that its gains extend to language generation, mathematical reasoning, and code generation.


\section{Preliminaries on graph-walks and MDMs}
\label{sec:prelim-mdlm}


\label{subsec:mdlm-graph-walks}
We use a masked diffusion model (MDM)~\citep{sahoo2024simple} (Appendix~\ref{app:prelim-related}) to generate order-$k$ random-walk samples $X_{1:L} \in [V]^L$, represented as vertex sequences on graph $G = (V, E)$.   Given an initial law $\rho_k$ on length-$k$ histories, and $k$-Markov transition kernel $P_k$, the random-walk distribution is $\nu_k(x_{1:L})
=
\rho_k(x_{1:k})
\prod_{\ell=k}^{L-1}
P_k(x_{\ell+1}\mid x_{\ell-k+1:\ell}).$ 

The model vocabulary is $V\cup\{\mask\}$. During standard MDM training, coordinates of $X_{1:L}\sim \nu_k$ are randomly masked, and the model is trained to predict the clean token at each masked position. The model never sees the graph, transition kernel, or graph features; it only observes sampled walks. Given revealed context $x_U$, the model returns one-coordinate marginals
$p_i(\cdot\mid x_U)$ which we interpret as an approximation to $\nu_k(X_i=\cdot\mid X_U=x_U)$, for
$i\in[L]\setminus U$.


Generation starts from an initial revealed set $U_0$. For unconditional generation, $U_0=\emptyset$. For conditional generation, we reveal a held-out path's endpoints and set $U_0=\{1,L\}$ with $\widetilde X_1=s$ and $\widetilde X_L=t$, targeting the endpoint-conditioned bridge law $\nu_k^{s,t} \defeq \nu_k(\cdot\mid X_1=s,X_L=t).$ At round $r$, let $M_r=[L]\setminus U_r$. Given $\widetilde X_{U_r}=x_{U_r}$, an unmasking policy chooses a block $B_r\subseteq M_r$, samples independently $\widetilde X_i\sim p_i(\cdot\mid x_{U_r}),
\; i\in B_r,$
and updates $U_{r+1}=U_r\cup B_r$.

The baseline policies differ only in how they choose $B_r$. Random unmasking chooses either one uniformly random coordinate from $M_r$, or a uniformly random subset of size $\min\{b,|M_r|\}$ in the $b$-at-a-time version. Greedy entropy chooses the smallest $H_i=-\sum_{a\in V}p_i(a)\log p_i(a)$; greedy confidence chooses the largest $C_i=\max_{a\in V}p_i(a)$; and greedy margin chooses the largest $\Delta_i=p_i^{(1)}-p_i^{(2)}$, where $p_i^{(1)}\ge p_i^{(2)}$ are the two largest probabilities under $p_i$.


\section{Understanding parallel unmasking via graph structure}
\label{sec:theoreitcal_parallel_sampling}
Simple graph constructions show that no single uncertainty-based parallel unmasking heuristic may be uniformly optimal. The key issue is not merely
which positions are most certain, but which positions can be revealed together
without introducing conditional-dependence errors. Adaptive one-at-a-time samplers help reduce approximation error when marginals are not perfectly learned \citep{kim2025train}. For parallel sampling, while such adaptive samplers may help when low-entropy positions act as separators, they can fail catastrophically by clustering updates within a dependent component and sampling coupled choices
independently. Figures~\ref{fig:star-dag-d3-m2} and~\ref{fig:kldag} illustrate these behaviors with two simple graph families.

We use the standard perfect conditional assumption to decouple the intrinsic dependence error of the parallel sampler from the model's approximation error.

\begin{assumption}[Perfect conditional marginals]
\label{ass:perfect-conditionals}
For every partial assignment $x_{U}$ satisfying
$\nu(X_{U}=x_{U})>0,$
the model's conditional marginal for the masked coordinates is exact, i.e, $\q_{i}(\cdot \mid x_{U})
=
\nu\!\left(
X_{i}=\cdot
\mid
X_{U}=x_{U}
\right)$ for all $i\in [L]\backslash U.$
\end{assumption}
Lemma~\ref{lem:any-order-exactness} in Appendix~\ref{app:any_order_unmasking} shows that sequential unmasking is order-invariant under Assumption \ref{ass:perfect-conditionals}. This follows directly from the probability product rule and is standard in any-order autoregressive modeling~\citep{anari2026parallel,kim2025train,uria2016nade,uria2014deep}.
For simplicity, we use DAGs with a context length such that each level corresponds to a position in the generated walk. We analyze two-at-a-time random ($\trum$) and entropy ($\teum$) samplers, breaking ties uniformly at random.

\begin{figure}[h]
    \captionsetup{skip=5pt}
\centering
\captionsetup{font=small}

\begin{minipage}[t]{0.30\textwidth}
\centering
\resizebox{\linewidth}{!}{%
\begin{tikzpicture}[
    xscale=1,
    yscale=0.5,
    >=stealth,
    every node/.style={
        circle,draw,thick,minimum size=5mm,inner sep=0pt
    },
    every edge/.style={draw,->,thick}
]

\node (r) at (0,4) {};

\node (a1) at (-1.4,2.5) {};
\node (a2) at ( 0,2.5) {};
\node (a3) at ( 1.4,2.5) {};

\node (b1) at (-1.4,0.8) {};
\node (b2) at ( 0,0.8) {};
\node (b3) at ( 1.4,0.8) {};

\path (r) edge (a1);
\path (r) edge (a2);
\path (r) edge (a3);
\path (a1) edge[>={Stealth[length=2mm,width=1.6mm]}] (b1);
\path (a2) edge[>={Stealth[length=2mm,width=1.6mm]}] (b2);
\path (a3) edge[>={Stealth[length=2mm,width=1.6mm]}] (b3);

\end{tikzpicture}%

}
\caption{\label{fig:star-dag-d3-m2}
Star-DAG with $d=3$ chains of length $m=2$.
}
\end{minipage}
\hfill
\begin{minipage}[t]{0.65\textwidth}
\centering
\resizebox{\linewidth}{!}{%
\begin{tikzpicture}[
    yscale=0.6,
    >=Stealth,
    bottleneck/.style={circle, draw, thick, fill=gray!20, minimum size=8mm, inner sep=0pt},
    internal/.style={circle, draw, thick, minimum size=5mm, inner sep=0pt}
]

\node[bottleneck] (u1) at (0,0) {$u_1$};
\node[bottleneck] (v1) at (3,0) {$v_1$};
\foreach \j/\y in {1/1.5,2/0.5,3/-0.5,4/-1.5} {
    \node[internal] (a1\j) at (1,\y) {};
    \node[internal] (b1\j) at (2,\y) {};
    \draw[->, thick] (u1) -- (a1\j);
    \draw[->, thick] (a1\j) -- (b1\j);
    \draw[->, thick] (b1\j) -- (v1);
}

\node[bottleneck] (u2) at (4.7,0) {$u_2$};
\node[bottleneck] (v2) at (7.7,0) {$v_2$};
\foreach \j/\y in {1/1.5,2/0.5,3/-0.5,4/-1.5} {
    \node[internal] (a2\j) at (5.7,\y) {};
    \node[internal] (b2\j) at (6.7,\y) {};
    \draw[->, thick] (u2) -- (a2\j);
    \draw[->, thick] (a2\j) -- (b2\j);
    \draw[->, thick] (b2\j) -- (v2);
}

\draw[->, thick] (v1) -- (u2);

\end{tikzpicture}%
}
\caption{\label{fig:kldag}
The $(K,L)$ bottleneck DAG, shown with $K=2$ bottleneck gadgets and $L=4$ parallel corridors.
}
\end{minipage}

\vspace{-12pt}
\end{figure}

\subsection{Star-DAG: parallel entropy beats random}

\begin{definition}\label{def:Star-DAG} Fix integers $d\ge 2$ and $m\ge 1$. Let $G({d,m})$ be the directed graph
with one root vertex $r$ and $d$ disjoint directed chains of length $m$
emanating from $r$:
\(
r \to v_{i,1}\to \cdots \to v_{i,m}.
\) for $i\in[d]$.
\end{definition}

\begin{restatable}{lemma}{entropybeatsrandom}
\label{lemma:entropybeatsrandom}
Let the context length be
\(
L=m+1
\) and Assumption \ref{ass:perfect-conditionals} hold.
Denote by $\coh^{\teum}$ and $\coh^{\trum}$ the probabilities of generating a coherent directed path in
$G(d, m)$ with
$\teum$, and $\trum$ samplers respectively.
Then
\[
\coh^{\teum}=1,\qquad \coh^{\trum}=
\frac{2}{m+1}
+
\frac{m-1}{d(m+1)}.
\]
\end{restatable}

The target path is determined by a single hidden chain index $I\in[d]$:
the root coordinate is deterministic, and every non-root coordinate reveals
the same index $I$. Lowest-entropy unmasking first selects the deterministic
root and one non-root coordinate, which identifies $I$ and makes all remaining
coordinates deterministic; random two-at-a-time unmasking succeeds only if its
first pair contains the root or if two independently sampled non-root
coordinates happen to choose the same chain. The proof is in Appendix~\ref{app:entropy_beats_random}.

\subsection{$(K,L)$ bottleneck DAG: parallel random beats entropy}
\begin{definition}\label{def:kldag}
    A $(K,L)$ bottleneck DAG is defined as a graph with $2K$ bottleneck nodes where each pair is connected by a corridor of $L$ parallel directed paths with two nodes each.
\end{definition}

\begin{restatable}{lemma}{randombeatsentropy}
\label{lemma:randombeatsentropy}
Let Assumption \ref{ass:perfect-conditionals} hold and $\coh^{\trum}$, $\coh^{\teum}$ denote the probabilities of a coherent path generation in the $(K,L)$ bottleneck DAG with \trum and \teum respectively. If $K$ is even and L > 1, then
  \bas{
\liminf_{K\rightarrow\infty}\bb{\coh^{\trum}-\coh^{\teum}}\geq \frac{2}{3}+\frac{1}{3L}.
}
\end{restatable}

In the bottleneck DAG, errors arise only when two positions from the same
dangerous corridor are revealed together but the independently sampled tokens
belong to different parallel paths. \trum spreads its pairs across all
masked positions and therefore rarely hits the same corridor, whereas
\teum first consumes the low-entropy bottleneck nodes and leaves
many corridor positions to be paired with each other; Appendix~\ref{app:random_beats_entropy}
formalizes this by comparing the resulting success recurrences.


These separations suggest judging a parallel unmasking rule by how each update affects the conditional dependencies among the still-masked positions. The aim is to make current and subsequent parallel updates distributionally correct. In the Star-DAG, entropy achieves this by first identifying which branch the path follows, whereas in the bottleneck DAG, random unmasking works better by spreading its choices out, reducing conflicting choices within a corridor. Marginal uncertainty is therefore only a proxy; the fundamental objective is to construct conditionally valid parallel updates.

A similar principle may apply to creative language tasks: a proof strategy or key lemma, a story twist, or a central thesis can constrain subsequent dependent details. Although structural anchors may divide the remaining generation into weakly coupled subproblems, naive parallel sampling may still produce locally inconsistent updates. For graph walks, the next section proposes parallel samplers that are provably correct for Markov chains.

\section{A new family of bisection samplers}
\label{sec:bisection_sampler}

\begin{figure}[ht]
    \captionsetup{skip=5pt}
    \centering
    \begin{tikzpicture}[
        x=1.12cm, y=0.46cm,
        token/.style={draw=black!20, rounded corners=1pt, minimum width=1.04cm,
            minimum height=0.36cm, inner sep=2pt, font=\scriptsize},
        masked/.style={token, dashed, text=black!45},
        revealed/.style={token, fill=black!6},
        fresh/.style={token, draw=blue!65!black, fill=blue!10, text=blue!65!black}
    ]
        \foreach \r in {1,2,3}
            \node[anchor=east, font=\scriptsize] at (0.4,1-\r) {Round \r};
        \foreach \i in {1,2,3,5,6,7}
            \node[masked] at (\i,0) {\mask};
        \node[fresh] at (4,0) {sat};
        \foreach \i in {1,3,5,7}
            \node[masked] at (\i,-1) {\mask};
        \node[fresh] at (2,-1) {small};
        \node[revealed] at (4,-1) {sat};
        \node[fresh] at (6,-1) {the};
        \foreach \i/\word in {1/The,3/cat,5/on,7/mat.}
            \node[fresh] at (\i,-2) {\word};
        \foreach \i/\word in {2/small,4/sat,6/the}
            \node[revealed] at (\i,-2) {\word};
    \end{tikzpicture}
    \caption{Bisection with $k=1$, each word is one token.
    Blue tokens are sampled in parallel at the midpoints of the remaining masked intervals;
    gray tokens were revealed in earlier rounds.}
    \label{fig:bisection-text}
\vspace{-10pt}
\end{figure}

Bisection sampling uses graph walks' Markov structure to split
generation into conditionally independent subproblems. For a first-order walk, $X_{<t} \perp\!\!\!\perp X_{>t}\mid X_t$, so revealing the
midpoint allows the two remaining intervals to be sampled recursively in
parallel. Figure~\ref{fig:bisection-text} illustrates this order on sample text.
For an order-$k$ walk, a contiguous block of $k$ tokens captures the
memory needed to separate the two sides. In each active masked interval of
length $\ell$, we therefore reveal the middle block of $\min\{k,\ell\}$ tokens
sequentially, then recurse on the remaining intervals in parallel. Middle blocks keep the recursion balanced; unrestricted
pivots near endpoints may unbalance the subproblems.
Algorithm~\ref{alg:k-bisection-sampling} in
Appendix~\ref{app:bisection_algorithms} gives the full procedure.

The following lemma formalizes conditional independence between blocks
in distinct masked intervals for endpoint-conditioned walks.

\begin{restatable}[Conditional independence for order-$k$ Markov bridges]{lemma}{kBridgeCI}
\label{lem:k-bridge-ci}
Let $1\le k\le L$, and let $\nu_k$ be the law of an order-$k$ Markov chain
on $X_{1:L}$. For endpoints $s,t$ with $\nu_k(X_1=s,X_L=t)>0$, let
$\nu_k^{s,t} := \nu_k(\cdot \mid X_1=s, X_L=t)$ be the bridge law. Let $U \subseteq [L]$ contain the endpoints, and suppose the
revealed coordinates in $U$ form contiguous separator chunks $I_0,\ldots,I_m$,
ordered from left to right, with $|I_j|\ge k$ for every internal chunk
$1\le j\le m-1$. Let $B_j$ be the masked interval between $I_{j-1}$ and
$I_j$, and let $S_j\subseteq B_j$ be any subset. Then, for every $x_U$ with
$\nu_k^{s,t}(X_U=x_U)>0$,
\[
\nu_k^{s,t}\!\left(
    X_{S_1}=x_{S_1},\ldots,X_{S_m}=x_{S_m}
    \mid X_U=x_U
\right)
=
\prod_{j=1}^m
\nu_k^{s,t}\!\left(
    X_{S_j}=x_{S_j}\mid X_U=x_U
\right).
\]
\end{restatable}

At each bisection level, the algorithm selects one separator block per masked
interval. Lemma~\ref{lem:k-bridge-ci} allows different blocks to be sampled in
parallel, while sequential sampling within each block preserves its internal
dependencies. Under Assumption~\ref{ass:perfect-conditionals}, each substep
therefore samples from the correct joint conditional law. The chain
rule over levels yields correctness. Each level uses at most $k$ sequential
substeps, and balanced splits give $\mathcal O(1+\log(L/k))$ levels. The next
lemma gives the guarantee for unconditional and endpoint-conditioned walks; proofs appear in
Appendix~\ref{app:bisection_sampling_analysis}.

\begin{restatable}[Exactness of order-$k$ bisection sampling]{lemma}{kBisectionExact}
\label{lem:k-bisection-exact}
Let $1\le k\le L$, and let $\nu$ be an order-$k$ random-walk law, either
unconditional or conditioned on fixed reachable endpoints. Let Assumption~\ref{ass:perfect-conditionals} hold. If
$\apxnu_k^{\mathrm{bis}}$ denotes the output law of
Algorithm~\ref{alg:k-bisection-sampling}, initialized with $U=\emptyset$ or
with the prescribed endpoints, respectively, then
$\apxnu_k^{\mathrm{bis}}=\nu$. Moreover, the parallel sampling depth is
$\mathcal O\!\left(k\log(L/k)\right)$.
\end{restatable}

\paragraph{Score-guided bisection sampling.}
\label{subsec:score_guided_bisection_sampling}
In practice, learned marginals need not satisfy
Assumption~\ref{ass:perfect-conditionals}. To address model approximation error,
we combine adaptive, score-guided selection with separators
that control dependence error. In each active
interval, we choose the lowest-uncertainty pivot in its middle half,
then grow a contiguous separator:  given the current active block, we grow the separator in the direction where uncertainty is the least.
At each substep, we compare the scores of the masked tokens immediately
adjacent to its left and right ends and reveal the lower-uncertainty token
(or the only available neighbor). We continue until the separator contains
$\min\{k,\ell\}$ tokens, for interval length $\ell$.\bk{} The middle-half restriction balances the recursion; contiguous
growth produces an order-$k$ separator whenever masked tokens remain on both
sides. The parallel depth remains
$\mathcal O(k\log(L/k))$. Each level uses at most $k$ sequential
denoiser calls, with all active intervals processed in parallel. Since the
pivot lies in the middle half, every remaining interval has length at most
$\lfloor3\ell/4\rfloor$. Thus there are $\mathcal O(\log(L/k))$ levels
before intervals have length at most $k$ and are filled in one final level.\bk{}
Algorithm~\ref{alg:k-score-guided-bisection} in
Appendix~\ref{app:bisection_algorithms} gives the full procedure.
Under Assumption~\ref{ass:perfect-conditionals}, score-guided bisection is
also exact for Markov chains: conditional
on the revealed history, each substep samples at most one coordinate per
previously separated interval, so the same argument
applies (Appendix~\ref{app:bisection_sampling_analysis}).

Section~\ref{sec:language_experiments} shows that these samplers also improve
speed-quality tradeoffs for language generation. This may be because
the MDM conditions on all revealed tokens at each recursive step, including
those in distant intervals, using long-range dependencies throughout sampling.

\section{Experiments}
\label{sec:experiments}
We use the graph-walk MDM setup from Section~\ref{subsec:mdlm-graph-walks} to
study how latent graph structure affects parallel samplers. In all
graph experiments, the vocabulary is the vertex set $V$ (and a $\mask$ token) with
\(
|V|=500.
\)
For each graph, we train an MDM on random-walk samples from the unconditional
path law \(\nu\). For order $k=1$, unless otherwise stated, we use lazy random walks on an undirected graph $G=(V, E)$, so that $P_1(v\mid u)=1/(\deg(u)+1)$ if $v=u$ or $(v,u) \in E$, and $0$ otherwise. All reported numbers in the graph experiments are averaged (with error bars) over five random runs. See Appendix \ref{app:coherence-tables} for details.

\paragraph{Graph families.}
ST--ER($p$) is a spanning tree augmented by independently adding each non-tree edge with probability $p$, which controls the graph density while preserving connectivity.



\paragraph{Metrics.}
\label{sec:exp-metrics}
Graph walks allow exact validation via both path coherence and TV distance, without an external judge model.
For an order-$k$ walk with initial law ${\rho_k}$ and transition kernel $P_k$,
we define the coherence of a sequence $\widetilde X_{1:L}$ as
\[
\operatorname{coh}_k(\widetilde X_{1:L})
=
\mathbf{1}\{{\rho_k}(\widetilde X_{1:k})>0\}
\prod_{\ell=k}^{L-1}
\mathbf{1}\!\left\{
P_k(\widetilde X_{\ell+1}\mid \widetilde X_{\ell-k+1:\ell})>0
\right\}.
\]
Average coherence is the fraction of coherent generated samples, a support-validity
check. Since coherence alone does not measure agreement with the transition
statistics, we also report row-weighted transition total variation. Let $\widehat P_k(\cdot\mid h)$ be
the empirical next-token distribution after history $h\in V^k$, and let
$\widehat\omega(h)$ be the empirical frequency of that history. We define
\[
\mathrm{TV}_k
=
\frac12
\sum_{h\in V^k:\widehat\omega(h)>0}
\widehat\omega(h)
\sum_{v\in V}
\left|
\widehat P_k(v\mid h)-P_k(v\mid h)
\right|.
\]
For endpoint-conditioned generation, we also report conditional TV. It compares empirical next-step transition probabilities with exact endpoint-conditioned probabilities, averaging errors across positions and endpoint pairs. See~\eqref{eq:conditional-tv} and Appendix~\ref{app:conditional-tv}.

\subsection{Unconditional generation}
\label{sec:exp-unconditional}
\begin{figure}
    \captionsetup{skip=5pt}
    \centering
    \includegraphics[width=\linewidth]{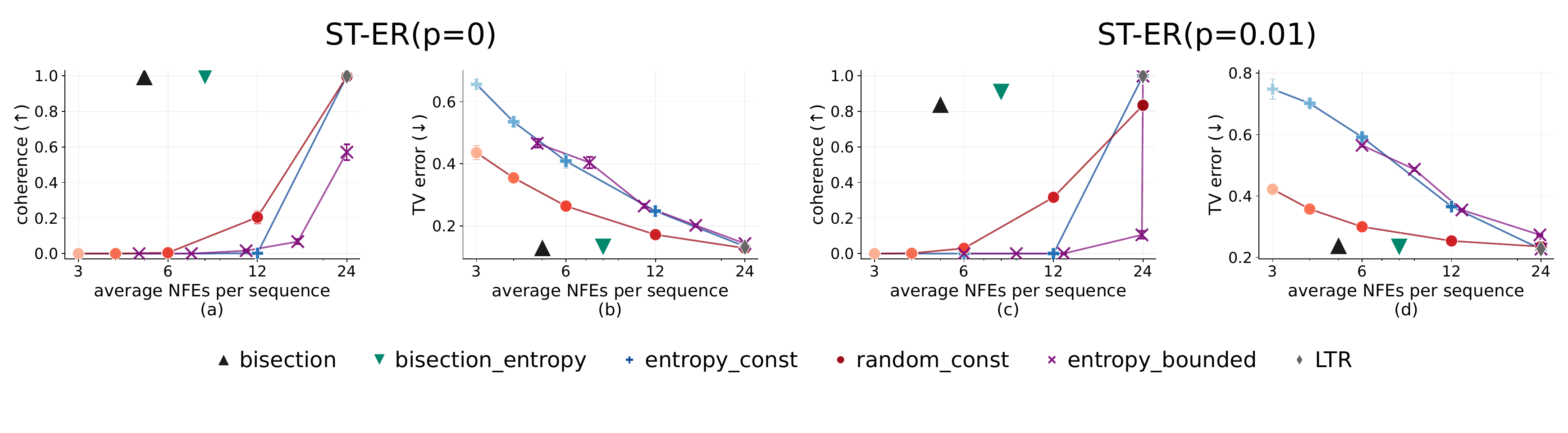}
    \vspace{-2em}
    \caption{
Unconditional random-walk generation on ST-ER graphs.
(a,b) report coherence and TV error on ST-ER($p=0$).
 (c,d) report the same metrics on ST-ER($p=0.01$).
}
    \label{fig:unconditional_coherence_mst}
    \vspace{-10pt}
\end{figure}

In the unconditional task, generation starts from the fully masked sequence and
targets the unconditioned walk law \(\nu\).
We evaluate unmasking policies on
ST--ER($p$) graphs.

Figure~\ref{fig:unconditional_coherence_mst} shows both metrics on ST-ER($p=0$), a spanning tree, and ST-ER($p=0.01$). As in Figure~\ref{fig:conditional_coherence}, changing only the edge density substantially changes the gap between random and entropy-based samplers in both metrics.

Moreover, bisection-style samplers substantially reduce the NFEs while preserving
coherence and transition fidelity, whereas  random or entropy-based parallel unmasking can lose coherence depending on the graph structure.

\begin{figure}[!htbp]
    \centering
    \begin{minipage}{0.55\textwidth}
    \centering
    \includegraphics[width=\linewidth]
        {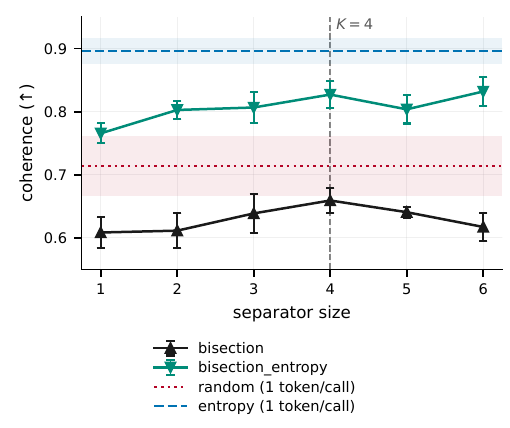}
    \caption{
        Coherence vs. separator size for order-$K=4$ walks (dashed line). Bars and shading show SD across four batches of $256$ samples.
    }
    \label{fig:distant-k-memory-sweep}
    \end{minipage}
\end{figure}

We also evaluate parallel samplers on order-$k$ unconditional random walks. Unlike first-order Markov random walks, we choose $X_{\ell+1}$ to depend only on its \textit{distant ancestors} in the
trajectory $X_{1:\ell}$ (see Appendix~\ref{app:exp-distant-k}). This more sharply tests preservation of higher-order structure:
a locally plausible sequence may still violate
the distant-history rule.

Figure~\ref{fig:distant-k-memory-sweep} shows that coherence improves as
the separator size approaches the true Markov order $K=4$.
Increasing the separator size from $1$ to $4$ raises mean coherence
from $0.608$ to $0.659$ for
divide-half bisection and from $0.766$ to $0.827$ for entropy-guided
bisection.
The random and entropy one-token baselines achieve $0.714$ and $0.896$,
respectively. 



\subsection{Endpoint-conditioned bridge generation}
\label{sec:exp-bridges}
\begin{figure}[!htbp]
    \captionsetup{skip=5pt}
    \centering
    \includegraphics[width=0.8\linewidth]{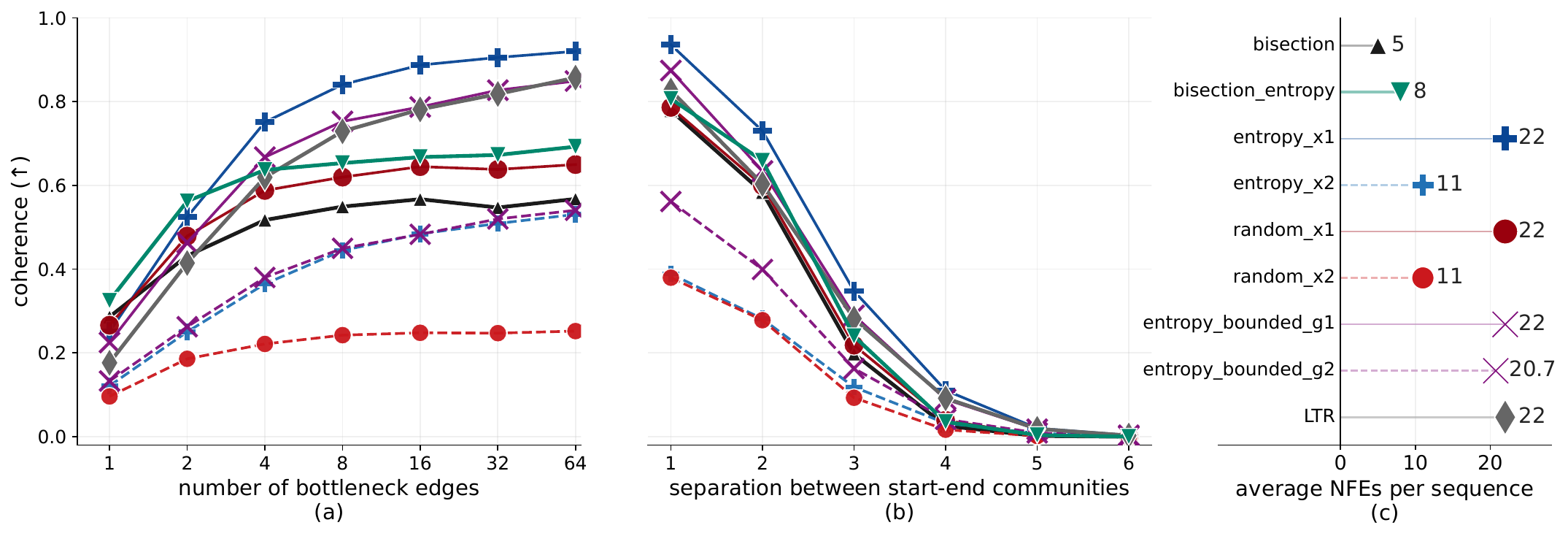}
    \caption{
Endpoint-conditioned bridge generation on bottleneck graph families.
 (a) shows the two-community bottleneck task
(b) shows the chain-of-communities task
(c) reports the average NFEs per sequence for each sampler.
}
    \label{fig:coherence_bottleneck_edges}
    \vspace{-12pt}
\end{figure}

For endpoint-conditioned generation, we reveal \(x_1\) and \(x_L\) from a held-out walk \(x_{1:L}\); the sampler fills in the bridge. For ST--ER($p$) graphs, Figure~\ref{fig:conditional_coherence}
reports TC and marginal approximation error.
Figure~\ref{fig:conditional_coherence_appendix} in
Appendix~\ref{app:coherence-tables} reports the corresponding coherence
and conditional TV results. We also use two bottleneck tests that probe compositional bridge generation.

\paragraph{Two-community bottleneck.}
The first setting has two components $V = V_1\sqcup V_2,
\;
|V_1|=|V_2|=250$, each sampled independently from the ST-ER$(p)$ family. We vary the number \(b\) of inter-component edges, training a separate model for each $b$. At test time, the revealed endpoints lie in different subgraphs. As \(b\) decreases, the bridge becomes harder since the sampler must discover a rarer crossing. This tests whether the sampler can compose local
motion within communities with the required inter-community crossing.

Figure~\ref{fig:coherence_bottleneck_edges}(a) shows sampler performance, with average NFEs in Figure~\ref{fig:coherence_bottleneck_edges}(c). With few bottleneck edges, one-at-a-time random and entropy samplers perform similarly; their ordering flips as more bottleneck edges are added. Entropy-guided bisection outperforms all baselines with few bottlenecks while using far fewer NFEs.

\paragraph{Chain of communities.}
We use ten communities $V=V_1\sqcup\cdots\sqcup V_{10}$, each with $|V_j|=50$ and sampled from the same ST--ER($p$) family. Consecutive communities are connected by a single edge $(r_j,r_{j+1})$, with $r_j\in V_j$ and $r_{j+1}\in V_{j+1}$, with no edges between non-consecutive communities.


We test endpoint-conditioned generation with endpoints in increasingly distant
communities. For separation $h$, we construct a walk $x_{1:L}$ with
$x_1\in V_j$ and $x_L\in V_{j+h}$ by joining the sequence
$r_j,r_{j+1},\ldots,r_{j+h}$ to two random walks confined to the endpoint
communities. Both local walks start at the designated vertices; we reverse the
left walk and choose their lengths so that the combined walk has $L$ vertices.
At test time, we reveal only $x_1$ and $x_L$ and mask all interior vertices.
This construction guarantees a valid bridge of the requested length.

A coherent bridge must traverse all $h$ inter-community edges, so increasing
$h$ requires the sampler to recover a longer sequence of hidden crossings.
Figure~\ref{fig:coherence_bottleneck_edges}(b) shows that one-at-a-time greedy
entropy achieves the highest coherence, whereas two-at-a-time greedy entropy
and random unmasking have comparable coherence. Bisection and entropy-guided
bisection achieve coherence close to one-at-a-time random unmasking, using
approximately $5$ and $8$ NFEs per sequence, respectively, compared with $22$
for the one-at-a-time baselines (Figure~\ref{fig:coherence_bottleneck_edges}(c)).

\section{Implications beyond graph random walks}
\label{sec:language_experiments}

 We next test whether the bisection schedules motivated by graph random walks
improve speed--quality tradeoffs for pretrained masked diffusion models
on mathematical reasoning, code generation, and unconditional text generation (see dataset details in Appendix~\ref{app:additional_language_experiments}).

\subsection{Reasoning benchmarks}
\paragraph{Setup.}
We evaluate parallel samplers on GSM8K \citep{cobbe2021training}, MBPP \citep{austin2021program} and HumanEval
\citep{chen2021evaluating} with Dream-v0-Instruct-7B \citep{ye2025dream7b} and LLaDA-1.5 \citep{zhu2025llada15},
as shown in Figure~\ref{fig:dream_llada_main}. All three
benchmarks are prompted zero-shot and decoded with $512$ output tokens at
temperature $0.1$ with top-$p=0.9$. We report $\mathrm{pass@1}$ with one sample per problem versus NFE. The DUS (dilated) sampler~\citep{luxembourg2025plan} uses semi-AR decoding.
\paragraph{Results.}


Figure~\ref{fig:dream_llada_main} shows that bisection offers a strong trade-off
between accuracy and NFEs on both models. Bisection performs well across all three tasks at roughly $30$ NFEs,
and attains the highest observed accuracy on the code-generation benchmarks,
MBPP and HumanEval. On GSM8K, bisection performs best among the
non-autoregressive samplers, while dilated semi-autoregressive decoding
achieves higher accuracy at larger sampling budgets. This may reflect the largely left-to-right structure of step-by-step mathematical solutions. These results suggest unmasking motivated by graph walks
can benefit reasoning and code generation with diffusion language models.


\begin{figure}
    \captionsetup{skip=5pt}
    \centering
    \includegraphics[width=\linewidth]{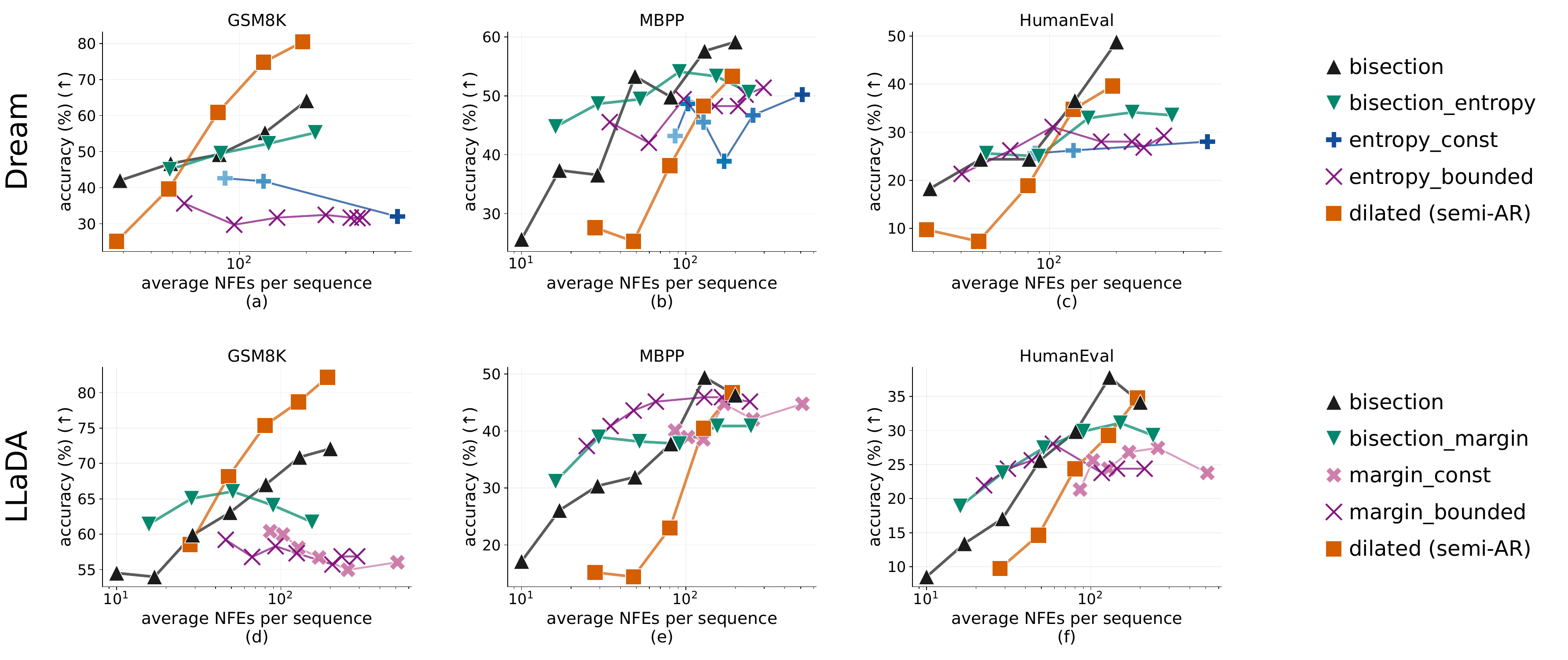}
    \caption{Reasoning benchmarks: $\mathrm{pass@1} (\uparrow)$ for Dream-v0-Instruct-7B (top) and LLaDA-1.5 (bottom) on GSM8K, MBPP and HumanEval, all prompted zero-shot.}
    \label{fig:dream_llada_main}
    \vspace{-12pt}
\end{figure}

\subsection{Unconditional language generation}

\begin{figure}[h]
    \captionsetup{skip=5pt}
    \centering
\includegraphics[width=\linewidth]{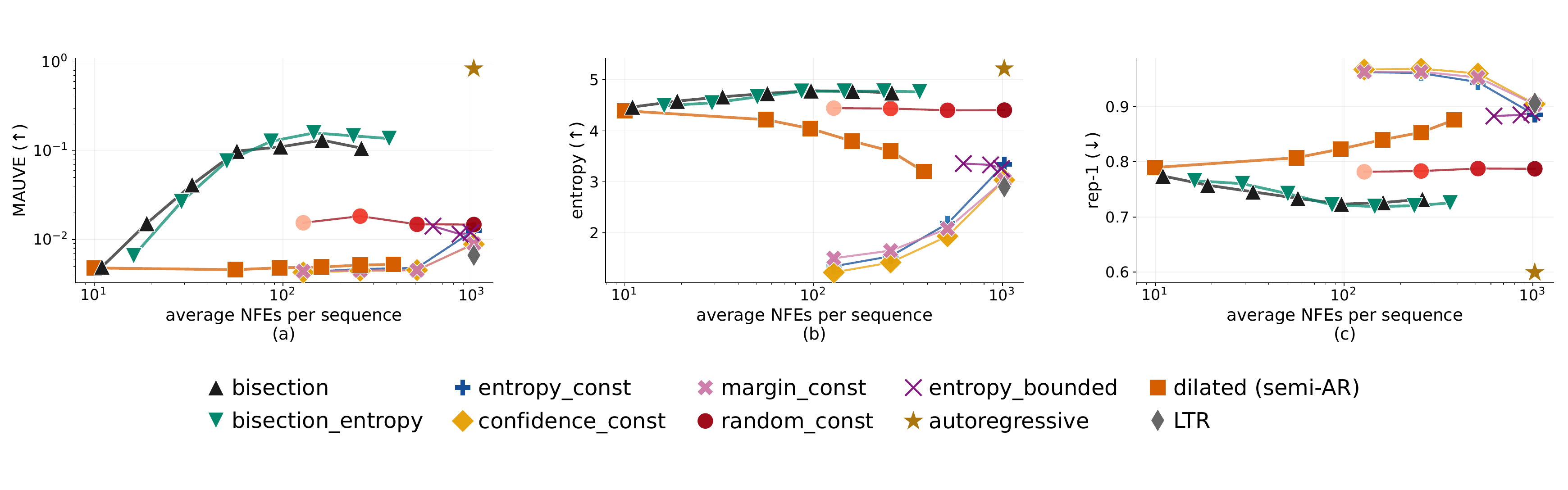}
\vspace{-2em}
    \caption{
    Speed-quality tradeoffs for language generation with a pretrained OpenWebText MDM
    across samplers: (a) MAUVE ($\uparrow$);
    (b) entropy ($\uparrow$); (c) token repetition rate ($\downarrow$).
    The x-axis reports average NFEs per sequence.
    }
\label{fig:language_experiments_main}
    \vspace{-12pt}
\end{figure}

\paragraph{Setup.}
We evaluate samplers on OpenWebText (OWT) \citep{Gokaslan2019OpenWeb}
using 1024 generated samples per method with top-$p=0.9$ sampling. We use the MDM
checkpoint trained on OWT \citep{sahoo2024simple, wang2025remasking}
and compare the samplers in Figure~\ref{fig:language_experiments_main}.
The x-axis reports average NFEs per sequence. We report MAUVE
\citep{pillutla2021mauve} with 2048 reference validation samples, token entropy and uni-gram repetition scores. MAUVE compares generated and reference text representations in a quantized embedding space; higher scores indicate closer agreement with the reference distribution. Generative perplexity and bi-/tri-gram repetition scores appear in Appendix~\ref{app:additional_language_experiments}. Higher is better for MAUVE and entropy,
while lower is better for repetition.

\paragraph{Results.}
Figure~\ref{fig:language_experiments_main} shows a clear compute--quality tradeoff.
The autoregressive baseline performs best overall, but uses roughly $1023$ NFEs per
sequence. Among non-autoregressive methods, bisection-based samplers perform best:
entropy-guided bisection reaches MAUVE $\approx 0.16$ at about $145$ NFEs, while
standard bisection reaches MAUVE $\approx 0.13$ at about $161$ NFEs. Both match or exceed the diversity of all other non-autoregressive baselines.

Low generative perplexity alone does not predict quality: Figure \ref{fig:language_experiments_appendix}  in Appendix~\ref{app:additional_language_experiments} shows that constant-entropy and
entropy-bounded samplers achieve low perplexity, but have poor MAUVE, consistent with
prior observations \citep{wang2025remasking}. Overall,
bisection, especially entropy-guided bisection, gives the strongest non-autoregressive
speed-quality tradeoff in the main-paper metrics, suggesting that coarse-to-fine unmasking based on exact conditional separation in graph walks also improves language generation. We show unconditional OpenWebText samples from the two best-performing bisection settings in Figure~\ref{fig:owt_samples} in Appendix~\ref{app:additional_language_experiments}.

\section{Conclusion and future directions}

We introduced graph random walks as a controlled benchmark for parallel MDM sampling. The hidden graph lets us evaluate support validity through coherence and distributional fidelity through transition TV. We showed that parallel unmasking depends on conditional independence rather than local uncertainty alone, and proposed bisection-style samplers using Markov separators for logarithmic parallel depth under perfect conditional marginals.

Future work includes quantifying denoiser estimation error propagation through sequential, parallel, and bisection samplers to extend our theory beyond exact conditional marginals. Further graph-walk stress tests include directed graphs, weighted kernels, nonreversible walks, time-inhomogeneous Markov chains, and larger community hierarchies. Our real-world experiments suggest graph walks can inform practical decoding for reasoning tasks and unconditional language generation. Future work includes scaling experiments and combining bisection with remasking or verifier-guided correction.


\section*{Acknowledgments}

SK gratefully acknowledges support from the Amazon AI PhD Fellowship. PS and CC acknowledge NSF grants 2217069 and CCF-2505865. We thank Dr.\ Vaishnavh Nagarajan (Google DeepMind) for discussions that deepened our understanding of planning and reasoning in language models. We also thank Dr.\ Adam Klivans and the Institute for Foundations of Machine Learning (IFML) at UT Austin for providing computing resources.

\section*{AI use statement}
In this work, we used generative AI tools to assist with writing and revising
proofs and to support experiment implementation. We also used these tools for
manuscript editing, literature searches and citation checking, code
development, and figure preparation. The authors take responsibility for the
final manuscript, mathematical claims, code, figures, and experimental results,
including content produced with the assistance of generative AI.

\bibliographystyle{alpha}
\bibliography{refs}

\clearpage
\appendix

\textbf{Appendix organization.}
The appendix is organized as follows.
\begin{itemize}
    \item \textbf{Appendix~\ref{app:prelim-related}}: Provides additional preliminaries and related work on masked diffusion models, any-order unmasking, and higher-order random walks on graphs.

    \item \textbf{Appendix~\ref{app:any_order_unmasking}}: Gives the deferred proof that exact one-token-at-a-time unmasking is order-invariant under perfect conditional marginals.

    \item \textbf{Appendix~\ref{app:entropy_beats_random}}: Proves the Star-DAG separation, where entropy-based two-at-a-time unmasking succeeds while random two-at-a-time unmasking can fail.

    \item \textbf{Appendix~\ref{app:random_beats_entropy}}: Proves the bottleneck-DAG separation, where random two-at-a-time unmasking outperforms entropy-based two-at-a-time unmasking.

    \item \textbf{Appendix~\ref{app:bisection_algorithms}}: Gives the full bisection algorithms.
    
    \item \textbf{Appendix~\ref{app:bisection_sampling_analysis}}: Contains the formal conditional-independence and exactness proofs for order-$k$ bisection sampling.

    \item \textbf{Appendix~\ref{app:additional_language_experiments}}: Provides dataset descriptions and additional language experiments on the OpenWebText dataset

    \item \textbf{Appendix~\ref{app:compute-resources}}: Reports compute resources and runtime details for the graph-walk training and sampler-evaluation experiments.

    \item \textbf{Appendix~\ref{app:existing-assets}}: Provides licenses for external codebase and dataset used.

    \item \textbf{Appendix~\ref{app:coherence-tables}}: Provides additional coherence tables for graph random-walk samplers across graph families, conditioning regimes, and unmasking policies.
\end{itemize}

\section{Additional preliminaries and related work}
\label{app:prelim-related}

The main text uses masked diffusion models through a black-box conditional
marginal view: given a partially revealed sequence $x_U$, the denoiser returns
one-coordinate distributions $p_i(\cdot\mid x_U)$ for masked coordinates
$i\notin U$. This appendix records the standard MDM details behind this view
and expands the related-work discussion.

\subsection{Additional details on masked diffusion models}

\paragraph{Absorbing corruption.}
Masked diffusion models for discrete data commonly use an absorbing-state
corruption process~\citep{austin2021structured,campbell2022continuous,
sahoo2024simple,shi2024simplified}. Let $\voc$ be the clean vocabulary and let
$\mask\notin \voc$ be the absorbing mask token. For a decreasing schedule
$\alpha_t\in[0,1]$ with $\alpha_0=1$ and $\alpha_1=0$, the forward process masks
coordinates independently:
\[
q_t(z\mid x)
=
\prod_{i=1}^L
\bb{
\alpha_t \delta_{x_i}(z_i)
+
(1-\alpha_t)\delta_{\mask}(z_i)
}.
\]
Thus $Z_t$ is exactly a partially revealed version of $X$, with revealed set
$U_t=\{i:Z_{t,i}\neq \mask\}$.

\paragraph{Reverse update.}
For $0\le s<t\le 1$, if $Z_{t,i}\neq \mask$, the absorbing coordinate is copied
unchanged. If $Z_{t,i}=\mask$, then the exact reverse posterior, conditioned on
the clean token, unmasks with probability
$\rho_{s,t}:=(\alpha_s-\alpha_t)/(1-\alpha_t)$ and otherwise remains masked:
\[
\Pr(Z_{s,i}=\cdot\mid Z_{t,i}=\mask,X_i)
=
\rho_{s,t}\delta_{X_i}(\cdot)
+
(1-\rho_{s,t})\delta_{\mask}(\cdot).
\]
Since $X_i$ is unknown at generation time, the learned sampler replaces
$\delta_{X_i}$ by a denoiser $p_{\theta,i}(\cdot\mid Z_t)$. Equivalently, a
finite reverse step first chooses a subset of currently masked coordinates and
then samples each chosen coordinate independently from its one-coordinate
denoising distribution. In the infinitesimal limit $s=t-dt$, $\rho_{s,t}=O(dt)$,
so for fixed $L$ the probability of revealing two or more coordinates in one
infinitesimal interval is $O(dt^2)$. This motivates viewing continuous-time
MDM sampling as a one-coordinate-at-a-time jump process, while finite-step or
accelerated samplers may reveal multiple coordinates from the same denoiser call.

\paragraph{Training objective.}
The usual MDM objective is a schedule-weighted denoising cross-entropy over
masked coordinates~\citep{sahoo2024simple,shi2024simplified}. In the notation of
the main text, a typical loss has the form
\[
\cL(\theta)
=
\E_{X\sim\nu,\;t,\;Z_t\sim q_t(\cdot\mid X)}
\sum_{i:Z_{t,i}=\mask}
w(t)\,
\bb{-\log p_{\theta,i}(X_i\mid Z_t)}
,
\]
for a nonnegative schedule-dependent weight $w(t)$. Our theoretical results do
not depend on the exact choice of $w(t)$; they only use the induced
conditional-marginal interpretation
$p_i(\cdot\mid x_U)\approx \nu(X_i=\cdot\mid X_U=x_U)$, formalized in
Assumption~\ref{ass:perfect-conditionals}.

\subsection{Related work}

\paragraph{Discrete and masked diffusion models.} Diffusion models for discrete data have been developed through several closely related formalisms. Early work introduced structured categorical corruptions and multinomial diffusion processes~\citep{austin2021structured,hoogeboom2021argmax}, while continuous-time formulations model the forward and reverse dynamics as jump processes or continuous-time Markov chains~\citep{campbell2022continuous,sun2023scorebased}. More recent approaches improve the learning objective or parameterization for discrete denoising, including score-entropy ratio estimation and simplified masked objectives~\citep{lou2024sedd,sahoo2024simple,shi2024simplified}. Discrete flow models and discrete flow matching give another view in which probability paths on finite state spaces are learned through posterior or denoising objectives~\citep{campbell2024generative,gat2024discrete}. Our work is complementary to these modeling advances: rather than proposing a new training objective, we study how a trained masked denoiser should be queried at inference time.

\paragraph{Diffusion language models.} Several works adapt diffusion or denoising ideas to text generation. Continuous-diffusion language models map tokens through continuous embeddings and enable controllable generation~\citep{li2022diffusionlm,gong2023diffuseq}, while masked or absorbing-state language diffusion models operate directly over token vocabularies~\citep{he2023diffusionbert,sahoo2024simple,shi2024simplified}. Recent large-scale diffusion language models and block diffusion models suggest that diffusion-style generation can scale to stronger language modeling regimes and interpolate between autoregressive and non-autoregressive generation~\citep{ye2023diffusionllm,arriola2025block,nie2025llada}. These models motivate our focus on sampling schedules: if MDMs are to be competitive at scale, one needs fast parallel samplers that do not destroy global sequence coherence.

\paragraph{Any-order generation and token ordering.} Masked language modeling began as a representation-learning objective in models such as BERT~\citep{devlin2019bert}, but the same ability to condition on arbitrary visible tokens also enables any-order generation. Any-order-autoregressive and masked generative models exploit this flexibility by choosing which coordinate to reveal next~\citep{uria2014deep,uria2016nade,shih2022anyorder}. In the exact one-token setting, every reveal order is valid by the chain rule, but practical orderings can substantially affect sample quality. MaskGIT popularized confidence-based iterative decoding~\citep{chang2022maskgit}, and recent work shows that token ordering matters sharply in masked diffusion models, including structured tasks such as Sudoku~\citep{kim2025train}. Our graph benchmark isolates the same phenomenon in a setting where the latent constraint structure is controllable and directly checkable.

\paragraph{Parallel and accelerated masked decoding.} The main computational appeal of masked diffusion is that multiple tokens can be revealed per denoiser call. Recent samplers accelerate generation by choosing larger blocks, adaptive block sizes, or by adding inference-time mechanisms such as KV caching, slow-fast schedules, auto speculation, and remasking~\citep{benhamu2025accelerated,wu2025fastdllm,wei2025slowfast,anari2026parallel,wang2025remasking}. These methods improve the speed-quality tradeoff, but a parallel update implicitly replaces a true block conditional by a product of one-coordinate marginals. Our contribution is to make this approximation explicit: on graph walks, whether a parallel update is safe is determined by conditional independence in the underlying sequence distribution.

\paragraph{Random walks and higher-order network dynamics.} Random walks and higher-order Markov models are standard tools for modeling sequential structure on graphs. Higher-order variants capture memory effects in web navigation, network flows, community detection, and complex networks \citep{chierichetti2012web,rosvall2014memory,krzakala2013spectral,benson2016higher,benson2017spacey}. We use these models in a different role: the walk distribution is a controlled sequence distribution for training and testing masked diffusion samplers. The graph and transition rule are hidden from the model but available to the evaluator, giving exact checks of support validity through coherence and, for first-order unconditional walks, distributional fidelity through transition TV.

\paragraph{Structured evaluation for generative models.} Open-ended text generation is difficult to evaluate with a single scalar metric; distributional metrics such as MAUVE compare generated and human text in an embedding space~\citep{pillutla2021mauve}, and our language experiments use OpenWebText as the evaluation domain~\citep{Gokaslan2019OpenWeb}. Graph walks serve a different purpose: they are not intended as a replacement for language evaluation, but as a mechanistic sandbox where failures can be attributed to support violations or transition-statistic errors. This makes them useful for diagnosing sampler behavior before transferring the resulting schedules back to language generation.

\section{Deferred Proofs from Section~\ref{sec:theoreitcal_parallel_sampling}}
\label{app:any_order_unmasking}

\begin{restatable}[Exactness of sequential unmasking]{lemma}{anyorderexactness}
\label{lem:any-order-exactness}
Let $\nu$ be a target distribution on paths $X_{1:L}=(X_1,\ldots,X_L)$, and let
$\sigma=(\sigma_1,\ldots,\sigma_L)$ be an ordering of $\{1,\ldots,L\}$.
Suppose Assumption~\ref{ass:perfect-conditionals} holds for this order $\sigma$.
Consider the sequential sampler
\(
\apx_{\sigma_r}\sim \q_{\sigma_r}(\cdot\mid \apx_{U_{r-1}})
\),
where $U_{r-1}:=\{\sigma_1,\ldots,\sigma_{r-1}\}$, for $r=1,\ldots,L$.
Let $\apxnu_\sigma$ denote the law of $\apx_{1:L}$. Then $\apxnu_\sigma=\nu$.
\end{restatable}
\begin{proof}
Fix $x_{1:L}\in \voc^L$. If $\nu(x_{1:L})>0$, then every prefix
$x_{U_{k-1}}$ has positive $\nu$-probability, so Assumption~\ref{ass:perfect-conditionals}
and the probability chain rule give
\bas{
\apxnu_\sigma(x_{1:L})
&=
\prod_{k=1}^L
\q_{\sigma_k}(x_{\sigma_k}\mid x_{U_{k-1}})  =
\prod_{k=1}^L
\nu(X_{\sigma_k}=x_{\sigma_k}\mid X_{U_{k-1}}=x_{U_{k-1}})
=
\nu(x_{1:L}).
}
If $\nu(x_{1:L})=0$, let $k$ be the first index such that
$\nu(X_{U_k}=x_{U_k})=0$. Then
$\nu(X_{U_{k-1}}=x_{U_{k-1}})>0$, while
$\nu(X_{\sigma_k}=x_{\sigma_k}\mid X_{U_{k-1}}=x_{U_{k-1}})=0$. By Assumption~\ref{ass:perfect-conditionals}, the corresponding
multiplicative term in the sampler probability is also zero: $p_{\sigma_r}(x_{\sigma_r}\mid x_{U_{r-1}})=0.$
Hence $\widetilde\nu_\sigma(x_{1:L})=0$.
\end{proof}

\section{A toy graph where entropy-two-at-a-time wins}
\label{app:entropy_beats_random}

\entropybeatsrandom*
\begin{proof}[Proof of Lemma~\ref{lemma:entropybeatsrandom}]
The target distribution is supported on the $d$ paths indexed by
$I\in[d]$. The first coordinate is deterministic:
\(
X_1=r.
\)
Every non-root coordinate reveals the same hidden chain index $I$:
\[
X_{t+1}=v_{I,t},\qquad t=1,\ldots,m.
\]

We first analyze entropy unmasking. Initially,
\(
H(X_1)=0,
\)
since $X_1$ is deterministic. For every non-root coordinate $t+1$,
\[
X_{t+1}\sim \mathrm{Unif}\{v_{1,t},\ldots,v_{d,t}\},
\]
so
\(
H(X_{t+1})=\log d.
\)
Therefore the greedy lowest-entropy two-at-a-time rule first selects the root
coordinate $X_1$ and one non-root coordinate $X_j$, $j\ge 2$.

The sampler draws $X_1=r$ deterministically. It also draws
\(
X_j=v_{i,j-1}
\)
for some $i\in[d]$. This reveals the hidden chain index $I=i$. Once $I=i$ is
known, every remaining coordinate is deterministic:
\[
X_{t+1}=v_{i,t},\qquad t=1,\ldots,m.
\]
Thus all later conditional marginals are point masses, and no inconsistency
can ever be introduced. Hence
\(
\coh^{\mathrm{TE}}_{d,m}=1.
\)

Now consider random two-at-a-time unmasking. The first random block is a
uniformly random two-element subset of the $L=m+1$ positions. The probability
that this block contains the root coordinate is
\[
\frac{L-1}{\binom{L}{2}}
=
\frac{2}{L}
=
\frac{2}{m+1}.
\]
If the first block contains the root, then the other selected coordinate
reveals the chain index $I$, after which all remaining coordinates are
deterministic. Thus this case succeeds with probability $1$.

On the complementary event, the first block contains two non-root coordinates.
Suppose the two selected positions are $a,b\ge 2$. Under the product-marginal
parallel update, the sampler draws
\[
X_a=v_{\widehat I_a,a-1},
\qquad
X_b=v_{\widehat I_b,b-1},
\]
where
\(
\widehat I_a,\widehat I_b
\overset{\mathrm{i.i.d.}}{\sim}
\mathrm{Unif}([d]).
\)
These two sampled coordinates are jointly extendable to a valid path if and
only if they come from the same chain:
\(
\widehat I_a=\widehat I_b.
\)
This happens with probability $1/d$.
If they disagree, no valid path in $G_{d,m}$ contains both sampled vertices,
so the sampler has already failed. If they agree, then the common value fixes
$I$, and all remaining coordinates are deterministic, so the sampler succeeds.
Therefore
\[
\coh^{\mathrm{TR}}_{d,m}
=
\frac{2}{m+1}
+
\left(1-\frac{2}{m+1}\right)\frac1d
=
\frac{2}{m+1}
+
\frac{m-1}{d(m+1)}.
\]
\end{proof}

\section{A toy graph where random-two-at-a-time wins}
\label{app:random_beats_entropy}

\randombeatsentropy*
\begin{proof}[Proof of Lemma~\ref{lemma:randombeatsentropy}]
Let $n$ denote the number of dangerous corridors remaining. Each dangerous
corridor contains a pair of internal positions. If these two positions are
unmasked in the same parallel round, then the independently sampled channel
choices agree with probability $1/L$ and disagree with probability $1-1/L$.
All other pairs of simultaneously unmasked positions are safe, because they
either lie in different corridors or contain a bottleneck position.

We first lower bound the success probability of \trum. Since the coordinate
choices of \trum are independent of the sampled values, the sequence of
two-at-a-time choices induces a uniformly random matching $M$ on the initially
masked positions. Let $m$ be the number of initially masked positions. For each
dangerous pair $D_i$, the probability that $D_i$ appears as an edge of $M$ is
$1/(m-1)$. Therefore, by a union bound,
\bas{
1-\coh_K^{\trum}
&=
\Pr\bb{\text{\trum fails}}  \\
&\le
\sum_{i=1}^{K}
\Pr\bb{D_i\in M}\Pr\bb{\text{channel mismatch}\mid D_i\in M} \\
&=
\frac{K}{m-1}\bb{1-\frac1L}.
}
The $(K,L)$ bottleneck DAG has at least $3K$ initially masked positions. Hence
\[
1-\coh_K^{\trum}
\le
\frac{K}{3K-1}\bb{1-\frac1L},
\]
and consequently
\[
\liminf_{K\to\infty} \coh_K^{\trum}
\ge
1-\frac13\bb{1-\frac1L}
=
\frac23+\frac{1}{3L}.
\]

We now upper bound the success probability of \teum. The lowest-entropy
coordinates are the bottleneck coordinates, so the first $K$ two-at-a-time
rounds reveal all $2K$ bottleneck nodes. After this, the only remaining
possible errors are the $K$ dangerous internal pairs. Let $\coh_n^{\teum}$ denote
the probability of generating a coherent sequence with entropy-based unmasking when $n$ dangerous pairs remain.
If the next entropy round selects the two positions from the same dangerous
pair, which occurs with probability $1/(2n-1)$, then it succeeds with
probability $1/L$ and leaves $n-1$ dangerous pairs. Otherwise it selects
positions from two different dangerous pairs, which occurs with probability
$(2n-2)/(2n-1)$, and leaves $n-2$ dangerous pairs. Therefore
\bas{
\coh_n^{\teum}
=
\frac{1}{2n-1}\cdot\frac{1}{L}\coh_{n-1}^{\teum}
+
\frac{2n-2}{2n-1}\coh_{n-2}^{\teum}.
}
Set $\coh_0^{\teum}=1$ and define
\[
H_n:=\max\{\coh_n^{\teum},\coh_{n-1}^{\teum}\}.
\]
Then
\bas{
\coh_{n-1}^{\teum}
&\le
\bb{
\frac{1}{2n-3}\cdot\frac{1}{L}
+
\frac{2n-4}{2n-3}
}H_{n-2},\\
\coh_n^{\teum}
&\le
\bb{
\frac{1}{2n-1}\cdot\frac{1}{L}
+
\frac{2n-2}{2n-1}
}H_{n-2}
=
\bb{1-\frac{1-1/L}{2n-1}}H_{n-2}.
}
Thus
\[
H_n
\le
\bb{1-\frac{1-1/L}{2n-1}}H_{n-2}.
\]
For $n=2r$, using $H_1\le 1$ and $H_2\le 1$,
\bas{
H_{2r}
&\le
\prod_{i=2}^{r}
\bb{1-\frac{1-1/L}{4i-1}} \\
&\le
\exp\bb{
-\frac{1-1/L}{4}\sum_{i=2}^{r}\frac1i
}
\le
\bb{\frac{2}{r+1}}^{(1-1/L)/4}.
}
In particular, $\coh_K^{\teum}\to 0$ along even $K$. Combining this with the lower
bound for \trum gives
\[
\liminf_{K\to\infty}
\bb{\coh_K^{\trum}-\coh_K^{\teum}}
\ge
\frac23+\frac{1}{3L}.
\]
\end{proof}

\section{Bisection sampling algorithms}\label{app:bisection_algorithms}

Algorithms~\ref{alg:k-bisection-sampling} and~\ref{alg:k-score-guided-bisection}
give the basic and score-guided bisection samplers, respectively. In both,
$1\le k\le L$. Initialize $U=\emptyset$ for unconditional sampling, or
$U=\{1,L\}$ with fixed reachable endpoint values for bridge sampling.
Scores depend only on the revealed context; ties are resolved by a fixed rule
or fresh randomness independent of the token draws.

\begin{algorithm}[h]
\caption{Order-$k$ bisection sampling}
\label{alg:k-bisection-sampling}
\begin{algorithmic}[1]
\Require Length $L$, order $k$, conditionals $\q_i(\cdot \mid \apx_U)$, initial unmasked set $U$
\While{$U \neq \{1,\ldots,L\}$}
    \State Let $B_1,\ldots,B_m$ be the maximal contiguous intervals in $[L]\setminus U$
    \ForAll{$B_j=\{a_j,\ldots,b_j\}$ in parallel}
        \State $\ell_j \gets |B_j|$ and $r_j \gets \min\{\ell_j,k\}$
        \State Choose the middle contiguous block $S_j=\{\tau_j,\tau_j+1,\ldots,\tau_j+r_j-1\}\subseteq B_j$
    \EndFor
    \For{$h=1,\ldots,k$}
        \ForAll{$j$ such that $h\le |S_j|$ in parallel}
            \State Let $i$ be the $h$th position in $S_j$
            \State Sample $\apx_i \sim \q_i(\cdot \mid \apx_U)$
        \EndFor
        \State Add all positions sampled in this substep to $U$
    \EndFor
\EndWhile
\State \Return $\apx_{1:L}$
\end{algorithmic}
\end{algorithm}

\begin{algorithm}[!htb]
\caption{Order-$k$ score-guided bisection sampling}
\label{alg:k-score-guided-bisection}
\begin{algorithmic}[1]
\Require Length $L$, order $k$, conditionals $\q_i$, scores $s_i$, initial unmasked set $U$

\Procedure{ParallelSample}{$\mathcal{C}, U$}

    \Comment{\texttt{$\mathcal{C}$ is a set of disjoint candidate index sets}}
    \State $I \gets \emptyset$
    \ForAll{$C \in \mathcal{C}$ \textit{in parallel}}
        \State Choose $i \in \argmin_{j \in C} s_j(\apx_U)$
        \State Sample $\apx_{i} \sim \q_{i}(\cdot \mid \apx_U)$
        \State $I \gets I \cup \{i\}$
    \EndFor
    \State $U \gets U \cup I$ \Comment{\texttt{Synchronous state update}}
    \State \Return $I, U$
\EndProcedure

\Statex
\Procedure{BisectionSample}{$L, k, U$}
    \While{$U\neq \{1,\ldots,L\}$}
        \State $\mathcal{B} \gets \{ B \mid B \text{ is a maximal contiguous masked chunk in } \{1,\dots,L\} \setminus U \}$

        \State \Comment{\texttt{Phase 1: Unmask centers}}
        \State $\mathcal{C}_{\text{center}} \gets \left\{ C \subseteq B \;\middle|\; B \in \mathcal{B}, \, C \text{ is the centered sub-interval of size } \lceil |B|/2 \rceil \right\}$
        \State $I, U \gets \Call{ParallelSample}{\mathcal{C}_{\text{center}}, U}$
        \State Initialize active separators $A_B \gets B \cap I$ for all $B \in \mathcal{B}$

        \State \Comment{\texttt{Phase 2: Grow separators outwards}}
        \For{$h=2,\ldots,k$}
            \State $\mathcal{B}_{\text{active}} \gets \{ B \in \mathcal{B} \mid |A_B| < |B| \}$

            \State $\mathcal{C}_{\text{grow}} \gets \{ \{\min A_B - 1, \max A_B + 1\} \cap B \mid B \in \mathcal{B}_{\text{active}} \}$
            \State $I, U \gets \Call{ParallelSample}{\mathcal{C}_{\text{grow}}, U}$ \Comment{\texttt{Grow blocks in parallel}}

            \ForAll{$B \in \mathcal{B}_{\text{active}}$}
                \State $A_B \gets A_B \cup (B \cap I)$ \Comment{\texttt{Update active separators}}
            \EndFor
        \EndFor
    \EndWhile
    \State \Return $\apx_{1:L}$
\EndProcedure
\end{algorithmic}
\end{algorithm}

\section{Deferred Proofs from Section~\ref{sec:bisection_sampler}}
\label{app:bisection_sampling_analysis}

\kBridgeCI*
\begin{proof}
The order-$k$ factorization is
\[
\nu_k(x_{1:L})
=\rho_k(x_{1:k})\prod_{\ell=k}^{L-1}
\nu_k(x_{\ell+1}\mid x_{\ell-k+1:\ell}).
\]
Each factor involves at most $k+1$ consecutive coordinates. Since every
internal revealed chunk has length at least $k$, no factor involves masked
coordinates from two distinct intervals $B_j$. Fixing $X_U=x_U$ therefore
leaves a product of nonnegative functions, one for each $x_{B_j}$. Its
normalizing constant factors over the intervals, giving
\[
\nu_k^{s,t}(X_{[L]\setminus U}=x_{[L]\setminus U}\mid X_U=x_U)
=\prod_{j=1}^m
\nu_k^{s,t}(X_{B_j}=x_{B_j}\mid X_U=x_U).
\]
The endpoint conditioning is already included in $x_U$, since $1,L\in U$.
Marginalizing each factor to $S_j$ proves the claim. This argument also applies
to the unconditional law, including masked intervals at the sequence
boundaries. With only one masked interval, the factorization is immediate.
\end{proof}

\kBisectionExact*
\begin{proof}
Let $U_r$ be the unmasked set at the beginning of bisection level $r$. Let
\(
B_1,\ldots,B_{m_r}
\)
be the maximal contiguous masked intervals. At initialization there is at
most one such interval; at later levels, every internal revealed separator
has length at least $k$. For each masked chunk $B_j$,
Algorithm~\ref{alg:k-bisection-sampling} chooses a contiguous block $S_j\subseteq B_j,
\; k_j:=|S_j|=\min\{|B_j|,k\}.$
The positions inside each $S_j$ are then unmasked sequentially, while different
chunks are processed in parallel.

By Lemma~\ref{lem:k-bridge-ci} and its unconditional version proved above,
conditional on the current unmasked variables
$X_{U_r}=x_{U_r}$, the selected blocks from different masked chunks are
conditionally independent:
\[
\nu
\left(
X_{S_1}=x_{S_1},\ldots,X_{S_{m_r}}=x_{S_{m_r}}
\mid X_{U_r}=x_{U_r}
\right)
=
\prod_{j=1}^{m_r}
\nu
\left(
X_{S_j}=x_{S_j}
\mid X_{U_r}=x_{U_r}
\right).
\]

Now fix a realization $x_{1:L}$ in the support of $\nu$. For each block $S_j$,
write its positions in the order used by the algorithm as $S_j=\{s_{j,1},\ldots,s_{j,k_j}\}.$
Because the blocks $S_j$ are conditionally independent given $X_{U_r}$,
conditioning additionally on already unmasked prefixes of these blocks preserves
independence across chunks. We will denote by a substep $h$ the sequential unmasking step of each $S_j$. Therefore, at each substep $h$, the next positions
\(
\{s_{j,h}: h\le k_j\}
\)
are conditionally independent given the current unmasked variables.

Let $U$ denote
\[
U := U_r \cup \bigcup_{j=1}^{m_r}
\{s_{j,t}:1\le t\le \min\{h-1,k_j\}\}.
\]
Hence, conditional on the current generated values $\apx_{U}=x_U$, the algorithm
samples the substep-$h$ positions with probability
\[
\prod_{j:\,h\le k_j}
\q_{s_{j,h}}(x_{s_{j,h}}\mid x_U).
\]
By the perfect conditional marginal assumption (Assumption~\ref{ass:perfect-conditionals}),
\[
\q_{s_{j,h}}(x_{s_{j,h}}\mid x_U)
=
\nu(X_{s_{j,h}}=x_{s_{j,h}}\mid X_U=x_U).
\]
Using the conditional independence from Lemma~\ref{lem:k-bridge-ci}, this product
is exactly the true conditional law of the positions sampled at substep $h$.

Thus, over all $h=1,\ldots,k$, the whole level samples exactly from the true
conditional law of the newly unmasked blocks:
\begin{align*}
&\Pr
\left(
\apx_{S_1}=x_{S_1},\ldots,\apx_{S_{m_r}}=x_{S_{m_r}}
\mid
\apx_{U_r}=x_{U_r}
\right)\\
&\qquad=
\nu
\left(
X_{S_1}=x_{S_1},\ldots,X_{S_{m_r}}=x_{S_{m_r}}
\mid
X_{U_r}=x_{U_r}
\right).
\end{align*}

After the level is completed, update
\[
U_{r+1}
=
U_r\cup S_1\cup\cdots\cup S_{m_r}.
\]
Applying the blockwise chain rule over all bisection levels gives
\[
\apxnu_k^{\mathrm{bis}}(x_{1:L})
=
\prod_r
\nu
\left(
X_{S_1^{(r)}}=x_{S_1^{(r)}},\ldots,
X_{S_{m_r}^{(r)}}=x_{S_{m_r}^{(r)}}
\mid
X_{U_r}=x_{U_r}
\right)
=
\nu(x_{1:L}).
\]
Therefore
\(
\apxnu_k^{\mathrm{bis}}=\nu.
\)

It remains to bound the number of rounds. At each bisection level, the algorithm
unmasks a contiguous block of at most $k$ positions in each active masked chunk.
The selected block lies in the middle of the chunk, so the largest remaining
masked chunk decreases by a constant factor until its length is at most $k$.
Thus the number of bisection levels is
\(
\mathcal O(1+\log(L/k)).
\)
Each level requires at most $k$ sequential substeps to unmask the selected block,
while different masked chunks are processed in parallel. Hence the total
parallel sampling depth is
\(
\mathcal O\!\left(k(1+\log(L/k))\right).
\)
\end{proof}

\paragraph{Score-guided bisection.}
The same argument proves exactness of
Algorithm~\ref{alg:k-score-guided-bisection}. Fix a level and condition on the
revealed history before a substep, including the index choices. These choices
depend only on revealed values and independent tie-breaking randomness, so
they reveal no additional information about masked values. The masked
intervals present at the start of the level remain conditionally independent
after conditioning on the values subsequently revealed within them. Each
substep chooses at most one coordinate in each of these intervals, so its
product of exact marginals equals the true joint conditional law. At the end
of the level, each new separator is contiguous and has length $k$, unless
its interval has been filled completely. The separator condition therefore
holds at the next level, and the chain-rule argument above gives the target
law. A pivot in the centered interval of size $\lceil\ell/2\rceil$ leaves at
most $\lfloor3\ell/4\rfloor$ positions on either side; growing the separator
can only shorten these intervals. Thus there are
$\mathcal O(1+\log(L/k))$ levels, each with at most $k$ sequential substeps.

\section{Additional details for real world examples}
\label{app:additional_language_experiments}

\paragraph{Datasets.}
OpenWebText (OWT) is a large corpus of web text collected from URLs shared on Reddit~\citep{Gokaslan2019OpenWeb}. GSM8K contains grade-school math word problems requiring multi-step reasoning~\citep{cobbe2021training}. MBPP consists of short Python programming problems specified in natural language~\citep{austin2021program}, while HumanEval contains hand-written programming problems used to evaluate functional correctness of generated code~\citep{chen2021evaluating}.

\paragraph{Additional metrics for language experiment} Figure~\ref{fig:language_experiments_appendix} reports generative perplexity under GPT2-Large, bi-gram and tri-gram
repetition rates for the same OpenWebText experiment. All bisection-based samplers,
especially entropy-guided bisection, maintain low repetition rates, while maintaining generative perplexity comparable to the AR model. Constant-entropy and
entropy-bounded samplers achieve low perplexity, but have poor MAUVE, and high repetition rates consistent with
prior observations \citep{wang2025remasking}.
These metrics support the same conclusion as the main-paper results: bisection-based
samplers provide the strongest non-autoregressive speed-quality tradeoff.

\begin{figure}[!hbt]
    \centering
    \includegraphics[width=\linewidth]{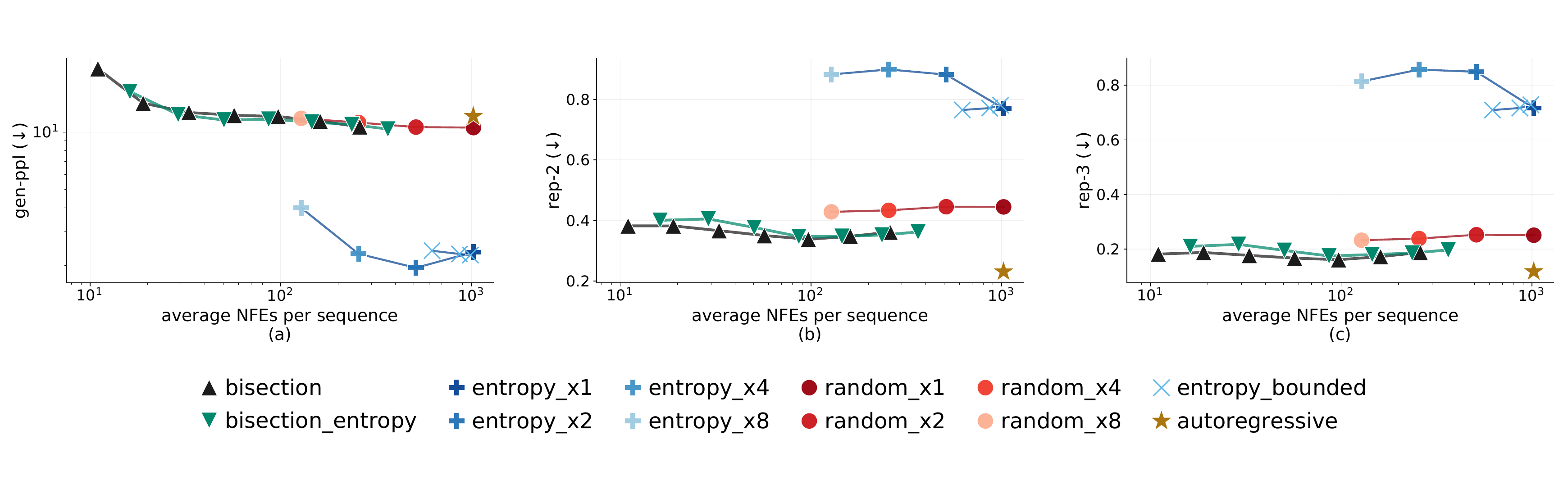}
    \caption{
    Additional language-generation metrics for the pretrained OpenWebText MDM.
    We report (a) generative perplexity under GPT2-Large ($\downarrow$),
    and (b) bi-gram repetition rate ($\downarrow$), (c) tri-gram repetition rate ($\downarrow$)
    The x-axis is the average NFEs per sequence.
    }\label{fig:language_experiments_appendix}
\end{figure}

\begin{figure}[t]
\footnotesize
\hrule\vspace{2pt}
\textbf{Bisection, $k{=}32$ (161 NFEs, MAUVE $0.131$)}
\vspace{2pt}\hrule
\begin{verbatim}
smartphones. As of the second quarter of 2017, nearly 6 million devices
have been connected to an Internet of Things (IoT) network, known as IoT
devices. Currently, it is estimated that as many as 20% of IoT devices
are connected to a home and will be affected by both the home automation
and command and control industries. Experts estimate that up to 40% of
home-connected devices are using telepresence, as well as telepresence
technology to remotely control sensors or other connected devices. It is
estimated that connected sensors account for a third of all IoT devices,
but that doesn't account for the other 80% of connected devices. The IoT
can also be used for a host of other products, such as medical devices
and gaming controllers. Many smart devices rely
\end{verbatim}
\hrule\vspace{2pt}
\textbf{Entropy-guided bisection, $k{=}16$ (146 NFEs, MAUVE $0.159$)}
\vspace{2pt}\hrule
\begin{verbatim}
that test was found to meet safety standards. The results were completed
over two weeks at the department's technology center where the test was
led by assistant director Jay Somentere. The test at KSTO was created in
the 1970s for the University of Oregon's Department of Applied Science.
"It's a safe and sound environment for both pedestrians and bicyclists,"
Somentere said. "Our goal is to be safe in the state of mind of every
pedestrian or bicyclist in Oregon. We're committed to it." Somentere
said his test had been using a three-axis sensor, but had to be updated,
at some point, for the use of an advanced digital imaging camera that
can detect the presence of multiple forms of movement in the roadway.
"There are two things that the test
\end{verbatim}
\hrule
\caption{Unconditional OpenWebText samples from the two best-performing
bisection settings, truncated to the first $150$ words.}
\label{fig:owt_samples}
\end{figure}
\section{Compute Resources}
\label{app:compute-resources}

All model training and sampler evaluation runs were performed on the TACC Vista
cluster.  GPU experiments used Vista Grace-Hopper nodes, each with one NVIDIA
H200 GPU with 96GB HBM3 memory, a 72-core NVIDIA Grace CPU, and approximately
120GB of host memory.  CPU-only preprocessing, aggregation, and plotting were
run on Vista CPU/login nodes; the node used for these auxiliary jobs exposed
144 ARM Neoverse-V2 cores and 237GiB of system memory.  These CPU-only jobs were
negligible compared with model training and sampler evaluation.

Each 100k-sample graph model trained in roughly 30 minutes on a single Vista GPU
node, i.e., well under one GPU-hour per model.  Sampler evaluation time varied
with the conditioning task.  All-sampler unconditional sweeps were the shortest,
taking roughly 10 minutes per model to generate all samples for each sampler.
Endpoint-conditional sweeps took roughly 25 minutes for each
ST--ER graph and about two hours for the bottleneck-family sweep, for roughly
three GPU-hours total across the final table settings.  Bottleneck-conditional
sweeps were the most expensive: the final bottleneck-family sweeps took between
about 30 minutes and 1.5 hours each, totaling roughly five GPU-hours for the
reported bottleneck-conditional table.

\section{Licenses}
\label{app:existing-assets}

The MDM
codebase and the pretrained \texttt{kuleshov-group/mdlm-owt} checkpoint are
released under the Apache-2.0 license.  The Hugging Face OpenWebText dataset
card lists the dataset packaging under CC0-1.0, while noting that the curators
do not own the underlying web text; the Hugging Face GPT-2 Large model card
lists the model under the MIT license.

\providecommand{\pmv}[2]{#1 $\pm$ #2}
\definecolor{bestcoh}{HTML}{D9EAD3}
\definecolor{secondcoh}{HTML}{FFF2CC}
\providecommand{\best}[1]{\cellcolor{bestcoh}\textbf{#1}}
\providecommand{\second}[1]{\cellcolor{secondcoh}#1}
\providecommand{\cohtableformat}{%
  \scriptsize
  \setlength{\tabcolsep}{2.0pt}%
  \renewcommand{\arraystretch}{1.05}%
}

\section{Additional details for Graph Random-Walk Samplers}
\label{app:coherence-tables}
\noindent\textbf{Notation.}
 All graphs are unweighted and have an additional self-edge. ST--ER($p$) denotes the spanning-tree plus Erd\H{o}s--R\'enyi graph family used
in the experiments: a spanning-tree backbone guarantees connectivity, and
each remaining non-tree edge is added independently with probability $p$.
Bottleneck graphs are composed of $m$ dense communities connected by rare
inter-community bridge edges; $b$ is the number of bridge edges per adjacent
community pair. Endpoint-conditional experiments condition only
on the prescribed start and terminal nodes of the walk.  Bottleneck-conditional
experiments further condition on the bottleneck-crossing constraints in the
bottleneck graph family, and therefore test whether the sampler can satisfy
both endpoint constraints and the required inter-community transitions.  In
sampler names, a suffix such as \texttt{\_x4} means that four masked positions
are updated per model call, while \texttt{\_exponential} uses a multiplicative growing update
budget.

\paragraph{ST-ER$\bm{(p)}$ graph construction.}
We begin by constructing a spanning-tree backbone. Starting from a
single root vertex, we grow a connected set one vertex at a time: at each step,
we attach one not-yet-added vertex to one vertex already in the connected set, both chosen uniformly at random.
After \(|V|-1\) steps, this gives a spanning tree \(T\), and therefore guarantees
that the graph is connected. We then add each remaining non-tree edge
independently with probability \(p\).
The tree backbone guarantees connectivity, while the Erd\H{o}s-R\'enyi
edges controls the average degree.

\paragraph{Model training details.}
The graph experiments use a masked discrete diffusion model over tokenized
random walks.  The backbone is the DDiT configuration:
hidden size $256$, conditioning dimension $64$, $4$ transformer blocks,
$4$ attention heads, dropout $0.1$, untied input/output embeddings, and
sigma-scaled logits.  All graph models use the absorbing-state diffusion
parameterization with substitution loss, continuous time ($T=0$), a log-linear
noise schedule with $\sigma_{\min}=10^{-4}$ and $\sigma_{\max}=20$, antithetic
time sampling, and EMA decay $0.9999$.

All reported graph models are trained from scratch on synthetic random-walk
datasets generated by our code. The following data sizes and training budgets
apply to the coherence tables; the order-$K$ experiment uses the settings in
Section~\ref{app:exp-distant-k}. Each dataset has $N=500$ nodes, fixed walk
length $24$, $100{,}000$ training walks, and $1{,}000$ held-out test walks.
Optimization uses AdamW with learning rate $3\times 10^{-4}$, $\beta_1=0.9$,
$\beta_2=0.999$, $\epsilon=10^{-8}$, and weight decay $0$.  We train for
$50{,}000$ steps with global batch size $256$, gradient clipping at $1.0$,
bfloat16 precision, and a cosine decay schedule with linear warmup for the
first $10\%$ of training steps.  The warmup starts at learning rate $10^{-6}$
and the cosine schedule decays to minimum learning rate $10^{-6}$.  Validation
and checkpointing are run every $2{,}000$ training steps, and the checkpoint
selected for the table evaluations is the last checkpoint from each run.

\begin{table}[t]
\centering
\small
\caption{Graph data and training settings used for the coherence tables.  ST--ER
denotes a spanning-tree backbone plus independent random non-tree edges.}
\label{tab:graph-experiment-settings}
\begin{tabular}{lp{0.70\textwidth}}
\toprule
Setting & Value \\
\midrule
Model & DDiT model with hidden size $256$, conditioning dimension $64$, $4$ blocks, $4$ heads, sigma-scaled logits, untied embeddings \\
Diffusion & Absorbing-state MDM, substitution parameterization, continuous time, log-linear noise, antithetic time sampling \\
Training size & $100{,}000$ walks per graph family \\
Test/cache size & $1{,}000$ held-out walks per graph family \\
Walk kernel & Lazy Markov random walk on each graph family \\
Walk length & Fixed length $24$ nodes \\
Optimizer & AdamW, lr $3\times10^{-4}$, betas $(0.9,0.999)$, eps $10^{-8}$, wd $0$ \\
Schedule & Cosine decay with $5{,}000$ warmup steps, warmup/min lr $10^{-6}$ \\
Training steps & $50{,}000$ steps, global batch size $256$, bf16 precision \\
Regularization & Dropout $0.1$, EMA $0.9999$, gradient clip $1.0$ \\
Validation/checkpointing & Every $2{,}000$ steps; checkpoint monitor is validation walk coherence; tables use the last checkpoint \\
Evaluation & $512$ unconditional samples, or $512$ prompts with $32$ samples per prompt for conditional settings \\
ST--ER $p=0$ & $p=0.0$ \\
ST--ER avg. degree $7$ & $p=0.0100683294$ \\
Bottleneck $m=2$ & $b\in\{1,8,64\}$, edge probability $0.02$ \\
Bottleneck $m=10$ & chain communities, edge probability $0.106377551$ \\
\bottomrule
\end{tabular}
\end{table}

For the two-community bottleneck comparison, lower-bridge graphs are nested
subgraphs of higher-bridge graphs: the $b=64$ graph is generated first, the
$b=8$ graph is obtained by deterministically trimming bridge edges from the
$b=64$ graph, and the $b=1$ graph is obtained by trimming the $b=8$ graph.
This keeps the within-community graph structure fixed while varying only the
number of inter-community bridges.  The ten-community bottleneck experiment
uses one chain-of-communities graph and reports aggregate bottleneck-conditional
metrics as well as span-specific bottleneck-conditional metrics.

All entries report coherence as mean $\pm$ standard deviation.  Unconditional
experiments use 512 samples per sampler.  Conditional and bottleneck-conditional
experiments use 512 prompts with 32 samples per prompt.  Standard deviations are
computed by splitting samples or prompts into four equal groups and taking the
standard deviation of the four group means.  The ST--ER average-degree-$7$
entries use the $100$k-training-sample model from the primary experiments, not
the sample-scaling $12.5$k model.
\subsection{Decomposition of the two sources of error for random walks} \label{app:error-decomposition}

This section illustrates how we compute the two sources of error in \eqref{eq:error-decomposition}, i.e., the marginal approximation error and the dependence error, for first order random-walks. All expectations are over target samples
$X \sim \nu$ and the randomness of the block rule. Throughout this section,
$p_i(\cdot \mid x_{U_r})$ denotes the network's output at position $i$
restricted to $V$. Concretely, we take the held-out test walks, run each sampler's block rule
on the network's outputs at the true context $X_{U_r}$ exactly as during
generation, and then reveal the true values $X_{B_r}$ rather than sampling.
\paragraph{Exact conditional marginals.}
Fix a round $r$ and context $x_{U_r}$. Let $\delta_\ell(v) = \mathbf{1}\{v = x_\ell\}$
for $\ell \in U_r$ and $\delta_\ell \equiv 1$ for $\ell \in M_r$, and define
\begin{equation}
\begin{aligned}
\alpha_1(v) &= \rho(v)\,\delta_1(v), &
\alpha_{\ell+1}(v) &= \delta_{\ell+1}(v)\sum_{u}\alpha_\ell(u)\,P(v \mid u),\\
\beta_L(u) &= 1, &
\beta_\ell(u) &= \sum_{v} P(v \mid u)\,\delta_{\ell+1}(v)\,\beta_{\ell+1}(v).
\end{aligned}
\label{eq:app:fb}
\end{equation}
By the Markov property, $\alpha_\ell(v) = \nu(X_\ell = v,\ X_{U_r \cap [\ell]} = x_{U_r \cap [\ell]})$
and $\beta_\ell(u) = \nu(X_{U_r \cap \{\ell+1,\dots,L\}} = x_{U_r \cap \{\ell+1,\dots,L\}} \mid X_\ell = u)$,
so for every masked $i \in M_r$,
\begin{equation}
\gamma_i(v) \;:=\; \nu(X_i = v \mid x_{U_r}) \;=\;
\frac{\alpha_i(v)\,\beta_i(v)}{\sum_{w}\alpha_i(w)\,\beta_i(w)} .
\label{eq:app:gamma}
\end{equation}

\paragraph{Marginal approximation error.}
With $\gamma_i$ from \eqref{eq:app:gamma} and $p_i$ the model marginal at the same
context,
\begin{equation}
\sum_{i \in B_r}\Delta_\theta(i; x_{U_r})
\;=\; \sum_{i \in B_r}\KL\big(\gamma_i \,\big\|\, p_i(\cdot \mid x_{U_r})\big)
\;=\; \sum_{i \in B_r}\sum_{v \in V}\gamma_i(v)\big[\log\gamma_i(v) - \log p_i(v \mid x_{U_r})\big].
\label{eq:app:Delta}
\end{equation}
This is the only place the trained network enters. It is finite since
$p_i > 0$ on $V$, and it does not depend on how $B_r$ is grouped.

\paragraph{Conditional total correlation.}
$\TC_\nu(X_{B_r} \mid x_{U_r})$ depends only on $\nu$ and on which positions were
revealed together; the network plays no role. Call a maximal run of consecutive
masked positions a \emph{segment}; two masked positions lie in the same segment
iff no revealed position lies strictly between them.

\begin{lemma}[Pairwise reduction of $\TC_\nu$ for first-order walks]
\label{lem:app:tc}
Fix a context $x_{U_r}$ with $\nu(X_{U_r} = x_{U_r}) > 0$.
\begin{enumerate}
\item[(i)] Let $i_1 < i_2$ be masked positions in the same segment, $g = i_2 - i_1$,
and $P^{g}(v \mid u)$ the $g$-step transition probability. Then
\begin{equation}
\pi(v \mid u) \;:=\; \nu(X_{i_2} = v \mid X_{i_1} = u,\ x_{U_r})
\;=\; \frac{P^{g}(v \mid u)\,\beta_{i_2}(v)}{\beta_{i_1}(u)},
\label{eq:app:pairkernel}
\end{equation}
where $\beta_{i_1}(u) = \sum_{w} P^{g}(w \mid u)\,\beta_{i_2}(w).$
\item[(ii)] List the positions of $B_r$ in increasing order. Then
\begin{equation}
\begin{aligned}
\TC_\nu(X_{B_r} \mid x_{U_r})
\;&=\; \sum_{\substack{(i_1,i_2)\ \text{consecutive in } B_r\\ \text{same segment}}}
I(X_{i_1}; X_{i_2} \mid x_{U_r}),\\
I(X_{i_1}; X_{i_2} \mid x_{U_r}) \;&=\; \sum_{u}\gamma_{i_1}(u)\,\KL\big(\pi(\cdot \mid u)\,\big\|\,\gamma_{i_2}\big).
\end{aligned}
\label{eq:app:collapse}
\end{equation}
In particular $\TC_\nu(X_{B_r} \mid x_{U_r}) = 0$ whenever $B_r$ contains at most one
position per segment, as for our bisection sampler with $k=1$.
\end{enumerate}
\end{lemma}

\begin{proof}
(i) Summing $\nu$ over all coordinates other than $i_1$, $i_2$ and $U_r$, the
positions left of $i_1$ contribute $\alpha_{i_1}(u)$, those right of $i_2$
contribute $\beta_{i_2}(v)$, and the positions strictly between $i_1$ and $i_2$
are all masked, so summing over them composes $g$ copies of $P$. Hence
$\nu(X_{i_1} = u, X_{i_2} = v \mid x_{U_r}) \propto \alpha_{i_1}(u)\,P^{g}(v \mid u)\,\beta_{i_2}(v)$;
dividing by $\gamma_{i_1}(u) \propto \alpha_{i_1}(u)\beta_{i_1}(u)$ gives the
kernel, and iterating the backward recursion over the $g$ unclamped positions
gives the normaliser.

(ii) Conditionally on $x_{U_r}$, the factors $P(x_{\ell+1} \mid x_\ell)$ that
couple two different segments involve a revealed coordinate and are constants,
so distinct segments are independent and $\TC_\nu$ is additive over them. Within
a segment, the same computation as in (i) with $m$ retained positions
$i_1 < \dots < i_m$ gives a joint law
$\propto \alpha_{i_1}(u_1)\prod_{s}P^{i_{s+1}-i_s}(u_{s+1} \mid u_s)\,\beta_{i_m}(u_m)$,
a Markov chain, whose entropy telescopes as
$H(\gamma_{i_1}) + \sum_s H(X_{i_{s+1}} \mid X_{i_s})$. Substituting into
$\TC_\nu = \sum_s H(\gamma_{i_s}) - H(X_{i_1},\dots,X_{i_m})$ leaves
$\sum_s [H(\gamma_{i_{s+1}}) - H(X_{i_{s+1}} \mid X_{i_s})] = \sum_s I(X_{i_s}; X_{i_{s+1}})$.
The expression for $I$ in \eqref{eq:app:collapse} is the standard form
$I = \mathbb{E}_{u \sim \gamma_{i_1}}\KL(\pi(\cdot\mid u)\,\|\,\gamma_{i_2})$, valid
because $\gamma_{i_2} = \sum_u \gamma_{i_1}(u)\pi(\cdot \mid u)$.
\end{proof}

Figure~\ref{fig:conditional_coherence} reports the
average TC and marginal approximation errors aggregated over rounds, i.e., we sum \eqref{eq:app:Delta} and \eqref{eq:app:collapse} over rounds per walk and
average over the held-out walks. The following lemma identifies the sum with the left-hand side of
\eqref{eq:error-decomposition}. Let $p_\theta(x) := \prod_r\prod_{i \in B_r}p_i(x_i \mid x_{U_r})$
along the teacher-forced schedule of $x$; for a block rule that is deterministic
given the revealed values, which includes all greedy and bisection samplers,
this is exactly the sampler's output law $p_\theta(X)$ in \eqref{eq:error-decomposition}.

\begin{lemma}
\label{lem:app:identity}
For any target law $\nu$ on $V^{L}$ and any block rule that selects $B_r$ from
$x_{U_r}$ alone,
\begin{equation}
\mathbb{E}_{X \sim \nu}\big[\log\nu(X) - \log p_\theta(X)\big]
\;=\; \mathbb{E}_{X \sim \nu}\Big[\sum_r \sum_{i \in B_r}\Delta_\theta(i; X_{U_r})\Big]
+ \mathbb{E}_{X \sim \nu}\Big[\sum_r \TC_\nu(X_{B_r} \mid X_{U_r})\Big].
\label{eq:app:identity}
\end{equation}
If the rule is deterministic given the revealed values, the left side is
$\KL(\nu(X)\,\|\,p_\theta(X))$ and \eqref{eq:error-decomposition} holds with equality. If the
rule uses independent randomness, the sampler's law is a mixture over
schedules, and by convexity of $\KL$ in its second argument the left side of
\eqref{eq:error-decomposition} is at most the expectation of \eqref{eq:app:identity} over
that randomness.
\end{lemma}

\begin{proof}
Since $U_{r+1} = U_r \sqcup B_r$ and $B_r$ is a function of $x_{U_r}$, the
definition of conditional probability gives
$\log\nu(x) = \sum_r \log\nu(x_{B_r} \mid x_{U_r})$ for every $x$ in the support.
Subtracting $\log p_\theta(x)$ and taking the expectation under $\nu$, condition
on $X_{U_r}$ in the $r$-th term: given $X_{U_r} = x_{U_r}$ the block is fixed and
$X_{B_r} \sim \nu(\cdot \mid x_{U_r})$, so the term equals
$\KL\big(\nu(X_{B_r} \mid x_{U_r})\,\|\,\prod_{i \in B_r}p_i(\cdot \mid x_{U_r})\big)$.
Inserting $\prod_i \nu(X_i \mid x_{U_r})$ inside the logarithm splits this into
$\TC_\nu(X_{B_r} \mid x_{U_r}) + \sum_{i \in B_r}\Delta_\theta(i; x_{U_r})$, using
that the $i$-th marginal of $\nu(X_{B_r} \mid x_{U_r})$ is $\nu(X_i \mid x_{U_r})$.
\end{proof}

\subsection{Endpoint-conditional TV evaluation}
\label{app:conditional-tv}
Let $G=(V,E)$ be the graph, $P$ its random-walk transition
matrix, and $L$ the number of nodes in each walk.
For an endpoint condition $c=(s,e)$, the true conditional
transition kernel at position $t\in\{1,\ldots,L-1\}$ is
\[
P_t^c(i,j)
=
\Pr(X_{t+1}=j \mid X_t=i, X_L=e)
=
\frac{P_{ij}(P^{L-t-1})_{je}}
     {(P^{L-t})_{ie}},
\]
whenever the denominator is positive.

Given $M_c$ generated walks conditioned on $c$, define
\[
C_t^c(i,j)
=
\sum_{r=1}^{M_c}
\mathbf{1}\!\left\{
x_t^{(r,c)}=i,\;
x_{t+1}^{(r,c)}=j
\right\}\] \\
\[n_t^c(i)
=
\sum_{j\in V} C_t^c(i,j)\] \\
\[\widehat{P}_t^c(i,j)
=
\frac{C_t^c(i,j)}{n_t^c(i)}
\qquad \text{for } n_t^c(i)>0.
\]
For each observed row, its TV error is
\[
d_t^c(i)
=
\begin{cases}
\displaystyle
\frac{1}{2}\sum_{j\in V}
\left|\widehat{P}_t^c(i,j)-P_t^c(i,j)\right|,
& (P^{L-t})_{ie}>\varepsilon, \\[6pt]
1,
& (P^{L-t})_{ie}\leq\varepsilon,
\end{cases}
\]
where $\varepsilon=10^{-15}$ is the numerical cutoff used
to identify states that cannot reach the prescribed endpoint.

Let $\mathcal{C}$ be the set of distinct sampled endpoint
conditions. If condition $c$ appears $m_c$ times among $R$
endpoint draws from the empirical training endpoint
distribution, its weight is $\omega_c=m_c/R$.
The reported conditional TV error is
\begin{equation}\label{eq:conditional-tv}
\widehat{\mathrm{TV}}_{\mathrm{cond}}
=
\frac{
\displaystyle
\sum_{c\in\mathcal{C}}\omega_c
\sum_{t=1}^{L-1}
\sum_{\substack{i\in V\\n_t^c(i)>0}}
n_t^c(i)\,d_t^c(i)
}{
\displaystyle
\sum_{c\in\mathcal{C}}\omega_c
\sum_{t=1}^{L-1}\sum_{i\in V} n_t^c(i)
}.
\end{equation}
Thus, transition-row errors are weighted by generated
transition counts and endpoint-draw multiplicities.
Counts include transitions from incoherent generated walks.

For the heavy conditional-TV evaluation, we use $L=24$,
$R=64$ endpoint draws, and $M_c=12{,}288$ generated walks
per distinct endpoint condition.

\noindent\textbf{Trend summary.}
The results support the central claim that the efficiency--coherence
tradeoff is governed by the conditional-dependence structure induced by the
unmasking order.  Greedy constant-\texttt{x} and exponential schedules degrade sharply
because their score rules often select local clusters of tokens from a single
denoiser call.  This effect is especially pronounced for confidence-based
decoding, where locally confident neighboring positions are likely to be
conditioned on one another but are nevertheless updated in parallel.  In
contrast, entropy-guided bisection preserves high coherence while reducing
sequential depth: its recursive middle-out schedule separates the next revealed
positions across subintervals, limiting local dependence among simultaneous
updates.  This is reflected in the tables, where bisection entropy is the strongest accelerated samplers and consistently performs well compared
with methods that obtain speedup by updating several positions from the same
denoiser call.  The bottleneck experiments further show that
graph structure controls the difficulty of conditional generation: coherence is
lowest when inter-community crossings are rare ($b=1$) or must traverse longer
community chains ($m=10$), and improves as wider bottlenecks provide more valid
bridge choices.

\begin{table}[t]
\centering
\cohtableformat
\caption{Unconditional coherence across samplers and graph families.}
\label{tab:unconditional-coherence}
\resizebox{\textwidth}{!}{%
\begin{tabular}{lcccccc}
\toprule
\makecell[l]{Sampling\\method}
& \makecell{ST--ER\\$p=0$\\lazy $0.5$}
& \makecell{ST--ER\\avg. deg. $7$\\lazy $0.125$}
& \makecell{Bottleneck\\$m=2,b=1$}
& \makecell{Bottleneck\\$m=2,b=8$}
& \makecell{Bottleneck\\$m=2,b=64$}
& \makecell{Bottleneck\\$m=10$} \\
\midrule
bisection & \pmv{0.996}{0.004} & \pmv{0.777}{0.051} & \pmv{0.812}{0.062} & \pmv{0.840}{0.055} & \pmv{0.799}{0.029} & \pmv{0.887}{0.004} \\
bisection entropy & \pmv{0.996}{0.004} & \pmv{0.844}{0.031} & \pmv{0.869}{0.012} & \pmv{0.891}{0.018} & \pmv{0.830}{0.038} & \pmv{0.916}{0.024} \\
greedy\_entropy & \pmv{0.998}{0.003} & \pmv{1.000}{0.000} & \pmv{1.000}{0.000} & \pmv{1.000}{0.000} & \pmv{1.000}{0.000} & \pmv{1.000}{0.000} \\
greedy\_entropy\_x2 & \pmv{0.002}{0.003} & \pmv{0.000}{0.000} & \pmv{0.000}{0.000} & \pmv{0.000}{0.000} & \pmv{0.000}{0.000} & \pmv{0.000}{0.000} \\
greedy\_entropy\_x4 & \pmv{0.000}{0.000} & \pmv{0.000}{0.000} & \pmv{0.000}{0.000} & \pmv{0.000}{0.000} & \pmv{0.000}{0.000} & \pmv{0.000}{0.000} \\
greedy\_entropy\_x6 & \pmv{0.000}{0.000} & \pmv{0.000}{0.000} & \pmv{0.000}{0.000} & \pmv{0.000}{0.000} & \pmv{0.000}{0.000} & \pmv{0.000}{0.000} \\
greedy\_entropy\_x8 & \pmv{0.000}{0.000} & \pmv{0.000}{0.000} & \pmv{0.000}{0.000} & \pmv{0.000}{0.000} & \pmv{0.000}{0.000} & \pmv{0.000}{0.000} \\
greedy\_entropy\_exponential & \pmv{0.002}{0.003} & \pmv{0.000}{0.000} & \pmv{0.000}{0.000} & \pmv{0.000}{0.000} & \pmv{0.000}{0.000} & \pmv{0.000}{0.000} \\
random & \pmv{0.988}{0.012} & \pmv{0.783}{0.026} & \pmv{0.818}{0.028} & \pmv{0.830}{0.019} & \pmv{0.758}{0.017} & \pmv{0.879}{0.031} \\
random\_x2 & \pmv{0.146}{0.024} & \pmv{0.248}{0.013} & \pmv{0.199}{0.025} & \pmv{0.250}{0.034} & \pmv{0.258}{0.032} & \pmv{0.127}{0.037} \\
random\_x4 & \pmv{0.000}{0.000} & \pmv{0.023}{0.012} & \pmv{0.004}{0.004} & \pmv{0.008}{0.006} & \pmv{0.006}{0.003} & \pmv{0.004}{0.004} \\
random\_x6 & \pmv{0.000}{0.000} & \pmv{0.000}{0.000} & \pmv{0.000}{0.000} & \pmv{0.000}{0.000} & \pmv{0.004}{0.004} & \pmv{0.000}{0.000} \\
random\_x8 & \pmv{0.000}{0.000} & \pmv{0.000}{0.000} & \pmv{0.000}{0.000} & \pmv{0.000}{0.000} & \pmv{0.000}{0.000} & \pmv{0.000}{0.000} \\
random\_exponential & \pmv{0.320}{0.054} & \pmv{0.006}{0.006} & \pmv{0.004}{0.004} & \pmv{0.004}{0.007} & \pmv{0.002}{0.003} & \pmv{0.012}{0.009} \\
greedy\_confidence & \pmv{0.996}{0.004} & \pmv{0.998}{0.003} & \pmv{0.996}{0.004} & \pmv{0.992}{0.006} & \pmv{0.990}{0.006} & \pmv{0.926}{0.012} \\
greedy\_confidence\_x2 & \pmv{0.002}{0.003} & \pmv{0.000}{0.000} & \pmv{0.000}{0.000} & \pmv{0.000}{0.000} & \pmv{0.002}{0.003} & \pmv{0.000}{0.000} \\
greedy\_confidence\_x4 & \pmv{0.000}{0.000} & \pmv{0.000}{0.000} & \pmv{0.000}{0.000} & \pmv{0.000}{0.000} & \pmv{0.000}{0.000} & \pmv{0.000}{0.000} \\
greedy\_confidence\_x6 & \pmv{0.000}{0.000} & \pmv{0.000}{0.000} & \pmv{0.000}{0.000} & \pmv{0.000}{0.000} & \pmv{0.000}{0.000} & \pmv{0.000}{0.000} \\
greedy\_confidence\_x8 & \pmv{0.000}{0.000} & \pmv{0.000}{0.000} & \pmv{0.000}{0.000} & \pmv{0.000}{0.000} & \pmv{0.000}{0.000} & \pmv{0.000}{0.000} \\
greedy\_margin & \pmv{0.994}{0.003} & \pmv{0.973}{0.004} & \pmv{0.977}{0.020} & \pmv{0.971}{0.016} & \pmv{0.936}{0.019} & \pmv{0.877}{0.022} \\
greedy\_margin\_x2 & \pmv{0.006}{0.006} & \pmv{0.012}{0.007} & \pmv{0.004}{0.004} & \pmv{0.006}{0.010} & \pmv{0.010}{0.009} & \pmv{0.010}{0.010} \\
greedy\_margin\_x4 & \pmv{0.000}{0.000} & \pmv{0.000}{0.000} & \pmv{0.000}{0.000} & \pmv{0.000}{0.000} & \pmv{0.000}{0.000} & \pmv{0.000}{0.000} \\
greedy\_margin\_x6 & \pmv{0.000}{0.000} & \pmv{0.000}{0.000} & \pmv{0.000}{0.000} & \pmv{0.000}{0.000} & \pmv{0.000}{0.000} & \pmv{0.000}{0.000} \\
greedy\_margin\_x8 & \pmv{0.000}{0.000} & \pmv{0.000}{0.000} & \pmv{0.000}{0.000} & \pmv{0.000}{0.000} & \pmv{0.000}{0.000} & \pmv{0.000}{0.000} \\
\bottomrule
\end{tabular}
}
\end{table}

\begin{table}[t]
\centering
\cohtableformat
\caption{Endpoint-conditional coherence across samplers and graph families.}
\label{tab:conditional-coherence}
\resizebox{\textwidth}{!}{%
\begin{tabular}{lcccccc}
\toprule
\makecell[l]{Sampling\\method}
& \makecell{ST--ER\\$p=0$\\lazy $0.5$}
& \makecell{ST--ER\\avg. deg. $7$\\lazy $0.125$}
& \makecell{Bottleneck\\$m=2,b=1$}
& \makecell{Bottleneck\\$m=2,b=8$}
& \makecell{Bottleneck\\$m=2,b=64$}
& \makecell{Bottleneck\\$m=10$} \\
\midrule
bisection & \pmv{0.947}{0.012} & \pmv{0.620}{0.008} & \pmv{0.755}{0.004} & \pmv{0.773}{0.007} & \pmv{0.680}{0.012} & \pmv{0.836}{0.008} \\
bisection entropy & \pmv{0.956}{0.013} & \pmv{0.745}{0.006} & \pmv{0.818}{0.009} & \pmv{0.822}{0.007} & \pmv{0.748}{0.003} & \pmv{0.973}{0.004} \\
greedy\_entropy & \pmv{0.952}{0.011} & \pmv{0.927}{0.003} & \pmv{0.941}{0.041} & \pmv{0.968}{0.006} & \pmv{0.943}{0.001} & \pmv{0.973}{0.003} \\
greedy\_entropy\_x2 & \pmv{0.611}{0.022} & \pmv{0.593}{0.007} & \pmv{0.555}{0.022} & \pmv{0.560}{0.009} & \pmv{0.537}{0.006} & \pmv{0.479}{0.008} \\
greedy\_entropy\_x4 & \pmv{0.126}{0.013} & \pmv{0.000}{0.000} & \pmv{0.000}{0.000} & \pmv{0.000}{0.000} & \pmv{0.000}{0.000} & \pmv{0.000}{0.000} \\
greedy\_entropy\_x6 & \pmv{0.041}{0.003} & \pmv{0.000}{0.000} & \pmv{0.000}{0.000} & \pmv{0.000}{0.000} & \pmv{0.000}{0.000} & \pmv{0.000}{0.000} \\
greedy\_entropy\_x8 & \pmv{0.015}{0.003} & \pmv{0.000}{0.000} & \pmv{0.000}{0.000} & \pmv{0.000}{0.000} & \pmv{0.000}{0.000} & \pmv{0.000}{0.000} \\
greedy\_entropy\_exponential & \pmv{0.029}{0.004} & \pmv{0.000}{0.000} & \pmv{0.000}{0.000} & \pmv{0.000}{0.000} & \pmv{0.000}{0.000} & \pmv{0.000}{0.000} \\
random & \pmv{0.938}{0.015} & \pmv{0.712}{0.005} & \pmv{0.788}{0.013} & \pmv{0.800}{0.006} & \pmv{0.701}{0.005} & \pmv{0.850}{0.007} \\
random\_x2 & \pmv{0.736}{0.009} & \pmv{0.284}{0.007} & \pmv{0.327}{0.008} & \pmv{0.326}{0.004} & \pmv{0.289}{0.004} & \pmv{0.424}{0.006} \\
random\_x4 & \pmv{0.424}{0.013} & \pmv{0.039}{0.001} & \pmv{0.048}{0.006} & \pmv{0.049}{0.002} & \pmv{0.042}{0.003} & \pmv{0.105}{0.006} \\
random\_x6 & \pmv{0.219}{0.011} & \pmv{0.003}{0.001} & \pmv{0.005}{0.001} & \pmv{0.004}{0.001} & \pmv{0.004}{0.001} & \pmv{0.020}{0.001} \\
random\_x8 & \pmv{0.105}{0.003} & \pmv{0.000}{0.000} & \pmv{0.000}{0.000} & \pmv{0.000}{0.000} & \pmv{0.000}{0.000} & \pmv{0.004}{0.001} \\
random\_exponential & \pmv{0.399}{0.012} & \pmv{0.016}{0.002} & \pmv{0.012}{0.002} & \pmv{0.014}{0.002} & \pmv{0.011}{0.002} & \pmv{0.028}{0.002} \\
greedy\_confidence & \pmv{0.937}{0.008} & \pmv{0.917}{0.001} & \pmv{0.930}{0.035} & \pmv{0.951}{0.003} & \pmv{0.919}{0.004} & \pmv{0.914}{0.003} \\
greedy\_confidence\_x2 & \pmv{0.662}{0.017} & \pmv{0.111}{0.027} & \pmv{0.101}{0.023} & \pmv{0.099}{0.028} & \pmv{0.065}{0.002} & \pmv{0.161}{0.042} \\
greedy\_confidence\_x4 & \pmv{0.168}{0.011} & \pmv{0.000}{0.000} & \pmv{0.000}{0.000} & \pmv{0.000}{0.000} & \pmv{0.000}{0.000} & \pmv{0.000}{0.000} \\
greedy\_confidence\_x6 & \pmv{0.063}{0.008} & \pmv{0.000}{0.000} & \pmv{0.000}{0.000} & \pmv{0.000}{0.000} & \pmv{0.000}{0.000} & \pmv{0.000}{0.000} \\
greedy\_confidence\_x8 & \pmv{0.024}{0.006} & \pmv{0.000}{0.000} & \pmv{0.000}{0.000} & \pmv{0.000}{0.000} & \pmv{0.000}{0.000} & \pmv{0.000}{0.000} \\
greedy\_margin & \pmv{0.928}{0.010} & \pmv{0.882}{0.007} & \pmv{0.900}{0.036} & \pmv{0.922}{0.001} & \pmv{0.858}{0.002} & \pmv{0.868}{0.003} \\
greedy\_margin\_x2 & \pmv{0.523}{0.019} & \pmv{0.500}{0.020} & \pmv{0.516}{0.025} & \pmv{0.460}{0.054} & \pmv{0.466}{0.020} & \pmv{0.483}{0.023} \\
greedy\_margin\_x4 & \pmv{0.198}{0.011} & \pmv{0.016}{0.004} & \pmv{0.011}{0.002} & \pmv{0.003}{0.002} & \pmv{0.005}{0.002} & \pmv{0.055}{0.027} \\
greedy\_margin\_x6 & \pmv{0.089}{0.011} & \pmv{0.000}{0.000} & \pmv{0.000}{0.000} & \pmv{0.000}{0.000} & \pmv{0.000}{0.000} & \pmv{0.007}{0.003} \\
greedy\_margin\_x8 & \pmv{0.040}{0.009} & \pmv{0.000}{0.000} & \pmv{0.000}{0.000} & \pmv{0.000}{0.000} & \pmv{0.000}{0.000} & \pmv{0.000}{0.000} \\
\bottomrule
\end{tabular}
}
\end{table}

\begin{table}[t]
\centering
\cohtableformat
\caption{Bottleneck-conditional coherence across samplers and bottleneck graph families.}
\label{tab:bottleneck-conditional-coherence}
\begin{tabular}{lcccc}
\toprule
\makecell[l]{Sampling\\method}
& \makecell{Bottleneck\\$m=2,b=1$}
& \makecell{Bottleneck\\$m=2,b=8$}
& \makecell{Bottleneck\\$m=2,b=64$}
& \makecell{Bottleneck\\$m=10$} \\
\midrule
bisection & \pmv{0.253}{0.028} & \pmv{0.572}{0.010} & \pmv{0.587}{0.008} & \pmv{0.182}{0.052} \\
bisection entropy & \pmv{0.285}{0.037} & \pmv{0.703}{0.010} & \pmv{0.696}{0.004} & \pmv{0.284}{0.091} \\
greedy\_entropy & \pmv{0.172}{0.025} & \pmv{0.897}{0.002} & \pmv{0.928}{0.004} & \pmv{0.288}{0.090} \\
greedy\_entropy\_x2 & \pmv{0.084}{0.010} & \pmv{0.478}{0.007} & \pmv{0.541}{0.007} & \pmv{0.113}{0.028} \\
greedy\_entropy\_x4 & \pmv{0.000}{0.000} & \pmv{0.000}{0.000} & \pmv{0.000}{0.000} & \pmv{0.000}{0.000} \\
greedy\_entropy\_x6 & \pmv{0.000}{0.000} & \pmv{0.000}{0.000} & \pmv{0.000}{0.000} & \pmv{0.000}{0.000} \\
greedy\_entropy\_x8 & \pmv{0.000}{0.000} & \pmv{0.000}{0.000} & \pmv{0.000}{0.000} & \pmv{0.000}{0.000} \\
greedy\_entropy\_exponential & \pmv{0.000}{0.000} & \pmv{0.000}{0.000} & \pmv{0.000}{0.000} & \pmv{0.000}{0.000} \\
random & \pmv{0.224}{0.025} & \pmv{0.668}{0.007} & \pmv{0.656}{0.005} & \pmv{0.199}{0.060} \\
random\_x2 & \pmv{0.083}{0.010} & \pmv{0.248}{0.005} & \pmv{0.254}{0.005} & \pmv{0.091}{0.028} \\
random\_x4 & \pmv{0.009}{0.001} & \pmv{0.029}{0.001} & \pmv{0.031}{0.002} & \pmv{0.019}{0.004} \\
random\_x6 & \pmv{0.001}{0.001} & \pmv{0.002}{0.000} & \pmv{0.002}{0.001} & \pmv{0.002}{0.001} \\
random\_x8 & \pmv{0.000}{0.000} & \pmv{0.000}{0.000} & \pmv{0.000}{0.000} & \pmv{0.000}{0.000} \\
random\_exponential & \pmv{0.004}{0.001} & \pmv{0.012}{0.000} & \pmv{0.010}{0.001} & \pmv{0.009}{0.003} \\
greedy\_confidence & \pmv{0.160}{0.024} & \pmv{0.903}{0.004} & \pmv{0.911}{0.002} & \pmv{0.287}{0.090} \\
greedy\_confidence\_x2 & \pmv{0.015}{0.003} & \pmv{0.090}{0.008} & \pmv{0.082}{0.004} & \pmv{0.028}{0.010} \\
greedy\_confidence\_x4 & \pmv{0.000}{0.000} & \pmv{0.000}{0.000} & \pmv{0.000}{0.000} & \pmv{0.001}{0.001} \\
greedy\_confidence\_x6 & \pmv{0.000}{0.000} & \pmv{0.000}{0.000} & \pmv{0.000}{0.000} & \pmv{0.000}{0.000} \\
greedy\_confidence\_x8 & \pmv{0.000}{0.000} & \pmv{0.000}{0.000} & \pmv{0.000}{0.000} & \pmv{0.000}{0.000} \\
greedy\_margin & \pmv{0.170}{0.018} & \pmv{0.845}{0.005} & \pmv{0.846}{0.010} & \pmv{0.270}{0.090} \\
greedy\_margin\_x2 & \pmv{0.077}{0.007} & \pmv{0.272}{0.010} & \pmv{0.407}{0.004} & \pmv{0.093}{0.025} \\
greedy\_margin\_x4 & \pmv{0.002}{0.001} & \pmv{0.003}{0.000} & \pmv{0.004}{0.001} & \pmv{0.008}{0.003} \\
greedy\_margin\_x6 & \pmv{0.000}{0.000} & \pmv{0.000}{0.000} & \pmv{0.000}{0.000} & \pmv{0.000}{0.000} \\
greedy\_margin\_x8 & \pmv{0.000}{0.000} & \pmv{0.000}{0.000} & \pmv{0.000}{0.000} & \pmv{0.000}{0.000} \\
\bottomrule
\end{tabular}
\end{table}

\begin{figure}[htbp]
    \centering
    \includegraphics[width=\linewidth]
        {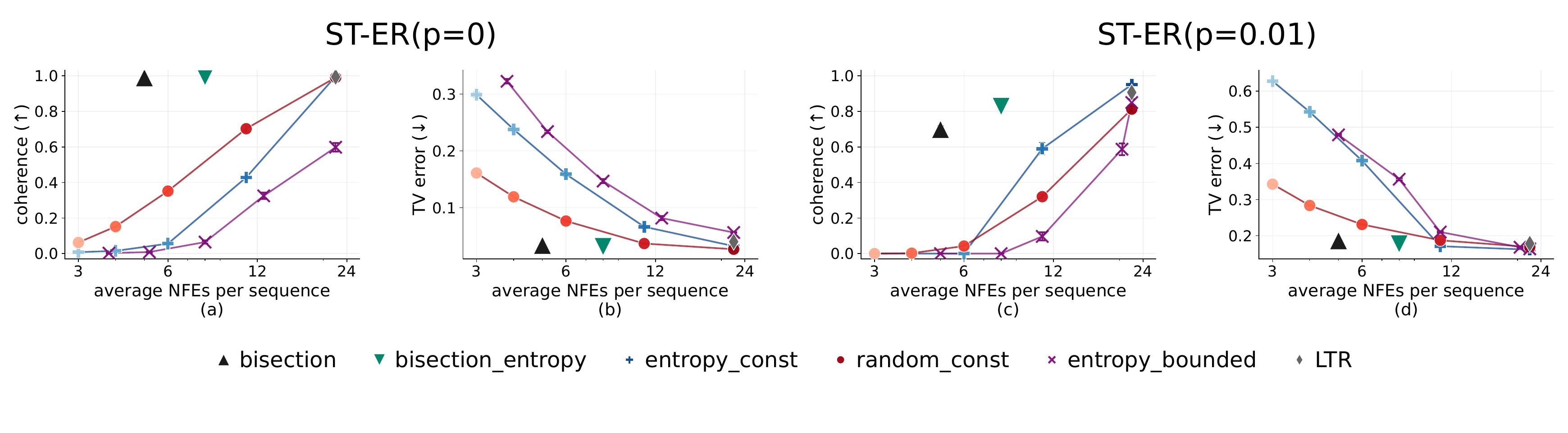}
    \caption{
    Coherence and TV error for endpoint-conditioned random-walk
    generation on ST--ER graphs.
    Higher coherence and lower TV error indicate better performance.
    }
    \label{fig:conditional_coherence_appendix}
\end{figure}

\clearpage
\subsection{Order $K$ random walks}
\label{app:exp-distant-k}

Let $G=(V,E)$ be the fixed undirected graph. For a vertex $u\in V$,
$N(u)=\{v\in V:(u,v)\in E\}$ denotes its set of neighbors, and $|N(u)|$
is its degree. The graph used below is connected and has no isolated vertices.
Here $K$ is the maximum history lag, and $\texttt{skip\_recent}$ is the
number of most recent lags excluded from the distant-history rule.
For a trajectory of length $L$, $X_\ell\in V$ is the random vertex at
position $\ell$, and $X_{1:\ell}=(X_1,\ldots,X_\ell)$ is its history.
We draw $X_1$ uniformly from $V$ and generate the remaining vertices using
the conditional transition law $P$ below. We write $\1\{\cdot\}$ for an
indicator and $|A|$ for the number of elements in a finite set $A$.
Given a history $X_{1:\ell}$, define the set of eligible lags
$A_\ell=\{j:\texttt{skip\_recent}+1 \le j \le K,\ j \le \ell-1\}$. If
$A_\ell \neq \emptyset$, the transition kernel is
\[
P(X_{\ell+1}=v \mid X_{1:\ell})
=
\frac{1}{|A_\ell|}
\sum_{j\in A_\ell}
\frac{\1\{v\in N(X_{\ell+1-j})\}}{|N(X_{\ell+1-j})|}.
\]
If $A_\ell=\emptyset$, we use the one-step non-lazy random-walk kernel
$P(X_{\ell+1}=v \mid X_{1:\ell})
=
\1\{v\in N(X_\ell)\}/|N(X_\ell)|$.
In our main experiment, $K=4$ and $\texttt{skip\_recent}=2$, so once enough
history is available, $A_\ell=\{3,4\}$. Thus $X_{\ell+1}$ is sampled from a
neighbor of either $X_{\ell-2}$ or $X_{\ell-3}$, while the two most recent
states are ignored by the data-generating rule.
The initialization uses $A_\ell=\emptyset$ for $\ell\leq 3$,
$A_4=\{3\}$, and $A_\ell=\{3,4\}$ for $\ell\geq 5$.

For evaluation, $\widetilde X_{1:L}=(\widetilde X_1,\ldots,\widetilde X_L)$
denotes a sequence produced by the learned sampler after removing padding
and other special tokens; $\widetilde X_\ell$ is its token at position $\ell$.
Here $L$ is the resulting sequence length, which may vary across generated
samples. We use the same coherence metric as above to check support validity,
specialized to the nonlocal transition rule.
A generated transition is valid if either $A_\ell=\emptyset$ and
$\widetilde X_{\ell+1}\in N(\widetilde X_\ell)$, or $A_\ell\neq\emptyset$ and
$\widetilde X_{\ell+1}\in N(\widetilde X_{\ell+1-j})$ for some $j\in A_\ell$.
Writing $\phi_\ell(\widetilde X_{1:L})$ for this indicator, we report $\operatorname{coh}_{\mathrm{dist}\text{-}K}(\widetilde X_{1:L})
=
\prod_{\ell=1}^{L-1}\phi_\ell(\widetilde X_{1:L}),$
averaged over generated samples. Empty unmasked sequences or sequences
containing tokens outside $V$ receive score zero. A repeated token is valid
only when it satisfies the same adjacency condition, since this experiment
uses no lazy transitions.

We use one ST-ER${(p)}$ graph with $|V|=500$ and $p=0.01$.
The $50{,}000$ training trajectories have lengths drawn uniformly from
$\{2,\ldots,24\}$ and are right-padded to $24$ positions, without beginning-
or end-of-sequence tokens. We use $1{,}000$ held-out trajectories for validation.
The model architecture and optimizer follow the graph-walk setup above;
for this experiment, we train for $15{,}000$ steps with batch size $256$
and select the checkpoint minimizing validation negative log-likelihood.
The graph and transition rule are not provided to the model.

We generate from $24$ masked positions and remove padding and other special
tokens before scoring the unmasked vertex sequence. We compare divide-half
bisection and entropy-guided bisection at separator sizes $1,\ldots,6$,
with entropy pivots restricted to the central half of each interval,
against one-token random and entropy baselines.
Each setting uses four independently seeded batches of $256$ samples from
the same fitted model. Error bars report the sample SD of the four batch
coherence means; they measure generation variability for this graph and model.

At separator size $4$, divide-half and entropy-guided bisection have mean
coherence $0.659$ and $0.827$, using $13$ and $15.19$ model calls per sequence,
respectively. Random one-token sampling has coherence $0.714$ and uses $24$
calls. Entropy one-token sampling has coherence
$0.896$ (SD $0.021$) and also uses $24$ calls.
Mean  unmasked sequence lengths range from $6.63$ to $7.16$ for divide-half bisection
and from $4.52$ to $5.21$ for entropy-guided bisection; they are $6.62$ and
$6.21$ for the random and entropy one-token baselines. These are support-validity
comparisons under each sampler's output-length distribution, rather than
comparisons at a fixed unmasked sequence length.

\end{document}